\documentclass[11pt]{jmlr} 
\jmlrpages{} 
\jmlrproceedings{}{} 

\usepackage{mathtools}
\LinesNumbered
\usepackage{algorithm}
\numberwithin{equation}{section}

\DeclareMathOperator*{\argmin}{argmin}
\def\reals{{\mathbb R}}
\def\naturals{{\mathbb N}}
\def\regret{{\mathrm Regret}}

\newcommand{\defeq}{\triangleq}
\def\regret{\ensuremath{\mbox{Regret}}}

\def\norm#1{\mathopen\| #1 \mathclose\|}

\newcommand{\abs}[1]{\left|#1\right|}
\newcommand{\ip}[1]{\langle #1 \rangle}

\newcommand{\newreptheorem}[2]{%
\newenvironment{rep#1}[1]{%
 \def\rep@title{#2 \ref{##1}}%
 \begin{rep@theorem}}%
 {\end{rep@theorem}}}
\newreptheorem{theorem}{Theorem}
\newreptheorem{lemma}{Lemma}
\newreptheorem{proposition}{Proposition}
\newreptheorem{claim}{Claim}
\newreptheorem{corollary}{Corollary}
\newreptheorem{mainlemma}{Main Lemma}

\usepackage{prettyref}
\newcommand{\pref}[1]{\prettyref{#1}}

\newcommand{\savehyperref}[2]{\texorpdfstring{\hyperref[#1]{#2}}{#2}}
\newrefformat{eq}{\savehyperref{#1}{\textup{(\ref*{#1})}}}
\newrefformat{eqn}{\savehyperref{#1}{Equation~\ref*{#1}}}
\newrefformat{lem}{\savehyperref{#1}{Lemma~\ref*{#1}}}
\newrefformat{def}{\savehyperref{#1}{Definition~\ref*{#1}}}
\newrefformat{thm}{\savehyperref{#1}{Theorem~\ref*{#1}}}
\newrefformat{cor}{\savehyperref{#1}{Corollary~\ref*{#1}}}
\newrefformat{sec}{\savehyperref{#1}{Section~\ref*{#1}}}
\newrefformat{app}{\savehyperref{#1}{Appendix~\ref*{#1}}}
\newrefformat{ass}{\savehyperref{#1}{Assumption~\ref*{#1}}}
\newrefformat{ex}{\savehyperref{#1}{Example~\ref*{#1}}}
\newrefformat{fig}{\savehyperref{#1}{Figure~\ref*{#1}}}
\newrefformat{alg}{\savehyperref{#1}{Algorithm~\ref*{#1}}}
\newrefformat{rem}{\savehyperref{#1}{Remark~\ref*{#1}}}
\newrefformat{conj}{\savehyperref{#1}{Conjecture~\ref*{#1}}}
\newrefformat{prop}{\savehyperref{#1}{Proposition~\ref*{#1}}}
\newrefformat{proto}{\savehyperref{#1}{Protocol~\ref*{#1}}}
\newrefformat{prob}{\savehyperref{#1}{Problem~\ref*{#1}}}
\newrefformat{claim}{\savehyperref{#1}{Claim~\ref*{#1}}}

\def\V{{\mathcal V}}
\def\W{{\mathcal W}}
\def\X{{\mathcal X}}

\def\bI{{\mathbf I}}

\def\be{{\mathbf e}}

\DeclareMathOperator{\argmax}{\arg\,\max}

\DeclareMathOperator{\diag}{\operatorname{diag}}
\DeclareMathOperator{\tr}{\operatorname{Tr}}
\newcommand{\identity}{\bI}
\newcommand{\Alg}{\operatorname{Alg}}
\newcommand{\wl}{\text{WL}}

\newcommand{\boostreg}{\text{BoostReg}}
\newcommand{\predict}{\text{Predict}}
\newcommand{\update}{\text{Update}}

\newcommand{\ind}{\mathbb{I}}

\newcommand{\NAME}{\text{FOLKLORE}\xspace}

\DeclarePairedDelimiter{\brk}{[}{]}
\DeclarePairedDelimiter{\crl}{\{}{\}}
\DeclarePairedDelimiter{\prn}{(}{)}

\title[Efficient Methods for Online Multiclass Logistic Regression]{Efficient Methods for Online Multiclass Logistic Regression}
\usepackage{times}
\usepackage{xspace}

\author{ 
 \Name{Naman Agarwal} \Email{namanagarwal@google.com}\\
 \Name{Satyen Kale} \Email{satyenkale@google.com}\\
 \Name{Julian Zimmert} \Email{zimmert@google.com}\\
 \addr Google Research}

\begin{document}

\maketitle

\begin{abstract}%
  Multiclass logistic regression is a fundamental task in machine learning with applications in classification and boosting.  
  Previous work \citep{foster2018logistic} has highlighted the importance of improper predictors for achieving ``fast rates'' in the online multiclass logistic regression problem without suffering exponentially from secondary problem parameters, such as the norm of the predictors in the comparison class. While \citet{foster2018logistic} introduced a statistically optimal algorithm, it is in practice computationally intractable due to its run-time complexity being a large polynomial in the time horizon and dimension of input feature vectors.
  In this paper, we develop a new algorithm, \NAME, for the problem which runs significantly faster than the algorithm of \citet{foster2018logistic} -- the running time per iteration scales quadratically in the dimension -- at the cost of a linear dependence on the norm of the predictors in the regret bound. This yields the first practical algorithm for online multiclass logistic regression, resolving an open problem of \citet{foster2018logistic}.
  Furthermore, we show that our algorithm can be applied to online bandit multiclass prediction and online multiclass boosting, yielding more practical algorithms for both problems compared to the ones in \citep{foster2018logistic} with similar performance guarantees. Finally, we also provide an online-to-batch conversion result for our algorithm.
\end{abstract}


\section{Introduction}

Logistic regression is a classical model in statistics used for estimating conditional probabilities that dates back to \citep{Berkson1944}. The model has been extensively studied in statistical and online learning and has been widely used in practice both for binary classification and multi-class classification in a variety of applications. 

In recent years, motivated by applications in clickthrough-rate prediction in large scale advertisement systems, the problem of {\em online} logistic regression has been intensively studied. This is a sequential decision-making problem where examples consist of a feature vector and a categorical class label. In each time step an online learner is provided a feature vector representing the example to be classified, and the learner is required to output a prediction of probabilities\footnote{To be precise, the learner outputs logits which induce a probability distribution via softmax.} of the label of the example. The learner's performance is measured by their {\em regret}, which is the difference between their cumulative log-loss, and the log-loss of the best possible linear predictor of bounded norm for the online sequence of examples computed in hindsight.

This problem is an instance of online convex optimization and hence the standard tools for deriving low regret algorithms, viz. Online Gradient Descent~\citep{zinkevich2003online} (OGD) and Online Newton Step~\citep{hazan2007logarithmic} (ONS) apply. However, the regret bounds derived using these methods are suboptimal. To discuss these bounds, let $d$ be the dimension of the input feature vectors (assumed to have length at most 1), $B > 0$ be a bound on (an appropriate) norm of the comparator linear predictor, and $T$ be the number of prediction rounds. Then OGD has a regret bound of $O(B\sqrt{T})$, whereas ONS has a regret bound of $O(de^{B}\log(T))$. While the latter bound has much better dependence on $T$ than the former, the exponential dependence on the $B$ parameter makes it worse for most practical settings of interest. 

In COLT 2012, \citet{mcmahan2012open} posed the open problem of whether it is possible to develop an algorithm with regret scaling polynomially in $B$ but logarithmically in $T$. This question was answered in the {\em negative} by \cite{pmlr-v35-hazan14a} who proved a lower bound showing that no {\em proper} online learning algorithm (i.e. one that generates a linear predictor to use in each round {\em before} observing the feature vector in that round) can achieve these desiderata. \citet{foster2018logistic} broke this lower bound via an {\em improper} online learning algorithm based on Vovk's Aggregating Algorithm~\citep{vovk1998game} which enjoys a regret bound of $O(d \log(BT))$, thus achieving a {\em doubly-exponential} improvement in the dependence on $B$ compared to ONS. A similar result was also obtained by \citet{kakade2005online} for the binary classification case. \citet{foster2018logistic} also proved a lower bound showing that this regret bound is {\em optimal}, even for improper predictors. Despite the optimality of the regret bound, this algorithm is unfortunately not very practical since its running time is a polynomial in $d$, $B$, and $T$ with high degrees. \citet{foster2018logistic} recognized the impracticality of their method, and state that ``obtaining a truly practical algorithm with a modest polynomial dependence on the dimension is a significant open problem.''
\citet{jezequel20a} solved this open problem for the {\em binary} classification setting, and developed a new improper online learning algorithm called AIOLI that has a suboptimal regret bound of $O(d B\log(BT))$, but much better running time: each round can be implemented to run in $O(d^2 + \log(T))$ time. Thus, AIOLI is a practical algorithm for online binary logistic regression. Similarly, for the offline (batch statistical, i.i.d.) setting for binary logistic regression, \citet{marteau-ferey} and \citet{mourtada} have developed learning algorithms with fast rates. The former develops a proper learning algorithm via generalized self-concordance techniques along with additional assumptions on the data distribution, whereas the latter develops an improper learning algorithm based on empirical risk minimization with an improper regularizer.

With the exception of \citep{foster2018logistic}, none of the papers mentioned above giving algorithms with fast rates and non-exponential dependence on $B$ extend to the practically important online {\em multiclass} logistic regression setting. In fact, it is possible to formally show that the style of analysis for the AIOLI algorithm \citep{jezequel20a} using Hessian dominance does not work for even 3 classes, see Section~\ref{sec:hessian-dominance-failure}. \citet{jezequel20a} also note the difficulty of extending the analysis to the multiclass setting and explicitly list developing a {\em practical} algorithm for the multiclass setting, and applying it to problems such as online bandit multiclass prediction and online multiclass boosting, as open problems. 

\begin{table}
\begin{center}
\begin{tabular}{ |c|c|c| } 
 \hline
 Algorithm & Regret & Running time (per step)\\
 \hline
 OGD & $B\sqrt{T}$ & $dK$ \\
 ONS & $de^{B}\log(T)$ & $d^2K^2$\\
 \cite{foster2018logistic} & $d\log(T)$ & $\mathrm{poly}(d,K,B,T)$\\ AIOLI  \citep{jezequel20a} [K=2] & $dB\log(T)$ & $d^2 + B\log(T)$\\ 
 GAF \citep{jezequel2021mixability} & $dBK \log(T)$ & $d^2K^3 + K^2T^3 $\\
 \NAME (This paper) & $dBK \log(T)$ & $d^2K^3 + K^2B \log(T) $\\
 \hline
\end{tabular}
\end{center}
\caption{Regret bounds and running time (in $\tilde{O}(\cdot)$) for relevant algorithms for online logistic regression. The running time bound for \cite{foster2018logistic} depends on the sampling algorithm and is a high-degree polynomial in $d,K,B,T$.}
\label{table:results}
\end{table}

In this paper, we solve these open problems. Specifically, if $K$ denotes the number of classes, we devise a new algorithm, \NAME (Fast OnLine K-class LOgistic REgression), that has a regret bound of $O(d(B + \log(K))K \log(T))$, and runs in $\tilde{O}(d^2K^3 + BK^2\log(T))$ time per round. Similarly to AIOLI, \NAME is also based on Follow-The-Regularized-Leader (FTRL) paradigm (see \citep{elad-oco} for background on FTRL) with quadratic surrogate losses and an improper regularizer which, in part, penalizes all the $K$ possible labels equally. The main distinction from AIOLI is the introduction of a {\em linear} term in the regularizer which is key to the analysis. The linear term is designed to ensure that a certain notion of instantaneous regret in any given time step is minimized. This crucially relies on the fact that there are only $K$ possible labels, and therefore, only $K$ possible values of the instantaneous regret. We choose the regularizer to be the solution to a minimax problem: minimze the maximum of the $K$ possible values of the instantaneous regret. It turns out that this minimax problem has a closed-form solution, which is used to design \NAME. Since the minimax problem relies on the knowledge of the feature vector in the current time step, \NAME becomes an {\em improper} online learning algorithm. An intriguing fact is that although \NAME does {\em not} reduce to AIOLI in the binary setting, the {\em probabilities} predicted by \NAME in the binary case exactly match those predicted by AIOLI (with some minor adjustments)  (see Section~\ref{sec:BinaryAIOLIRedux}).

We then employ \NAME as a subroutine in two problems of interest, also previously considered by \citet{foster2018logistic}: online bandit multiclass prediction, and online multiclass boosting. For the online bandit multiclass problem, we give a new reduction which transforms a regret bound in the full-information setting to the bandit setting. Using \NAME as the full-information online multiclass prediction algorithm, the resulting algorithm has a regret bound of $O(\sqrt{dK^2(B + \log(K))\log(BT)T})$, which is roughly an $O(\sqrt{B})$ factor worse than the regret bound of the OBAMA algorithm of \citet{foster2018logistic}, but is significantly more practical. For the online multiclass boosting problem, we give a new reduction to \emph{binary} logistic regression. Using \NAME (or AIOLI) as the online prediction algorithm results in a sample complexity of $O(T\log(K)/(N\gamma^2)+\log^2(T)/(\gamma^2)+KS/\gamma)$. Again, this is a factor $\log(T)$ worse than the AdaBoost.OML++ algorithm of \citet{foster2018logistic}, but significantly more practical.

After this manuscript initially appeared on arXiv, we were made aware of the independent and concurrent work \citep{jezequel2021mixability} which provides an efficient low-regret algorithm called Gaussian Aggregating
Forecaster (GAF) for $1$-mixable losses satisfying certain assumptions. In particular for multiclass logistic regression, GAF has the same regret bound of $O(dBK\log(T))$ as \NAME, but a per-step running time of $O(d^2K^3 + K^2T^3)$. GAF is based on the same lower-bound approximation as our algorithm, however it still requires sampling which leads to a $T$-dependent running time complexity per step. Our algorithm on the other hand is purely based on the FTRL paradigm and enjoys a much better running time. This and other relevant results for multiclass logistic regression are summarized in Table \ref{table:results}.

\section{Problem setting and notation}
\label{sec:prelims}

Let $\| \cdot \|_p$ denote the $\ell_p$ norm on $\reals^d$ and $\|\cdot\|_F$ as the Frobenius norm of a matrix. Further let $\| \cdot \|_{2,\infty}$ denote the $2 \rightarrow \infty$ norm of a matrix defined as 
\[ \|A\|_{2,\infty} \defeq \sup_{x: \|x\|_2 \leq 1} \|Ax\|_{\infty}\]

for any $K \in \naturals$, let $[K]$ be the set $[ 1, \ldots K]$. Define $\Delta_K \in \reals^K$ be the standard $K$-dimensional simplex representing distributions over $[K]$. For any positive definite matrix $A$, the norm $\|\cdot\|_A$ on $\reals^d$ is given by $\| x \|_A = \sqrt{\ip{x, Ax}}$. We will denote by $\otimes$ the standard Kronecker product between matrices. When using vectors $u\in \reals^{j_1}, v \in \reals^{j_2}$, $u \otimes v \in \reals^{j_1j_2}$ is defined as the Kronecker product obtained by interpreting $u,v$ as $j_1 \times 1, j_2 \times 1$ dimensional matrices respectively.  

In the multiclass learning problem the input is assumed to be from the set $\X = \{x \in \reals^d | \|x\|_2 \leq R\}$ for some $R \geq 0$. The number of output labels is $K$. A linear predictor parameterized by $W \in \reals^{K \times d}$, given an input $x \in \X$, assigns a score to every class as given by the vector $Wx \in \reals^K$. Given a $K \times d$ sized matrix $W$, we will denote by $\overrightarrow{W} \in \reals^{Kd}$ the $Kd$-dimensional vectorization of $W$ in a canonical manner. Let $W_k$ represent the $k^{th}$ row of the matrix $W$. The set of linear predictors we will be dealing with will be assumed to be $\W = \{W\in\reals^{d\times K} |  \|W\|_{2,\infty} \leq B\}$ for some $B \geq 0$. Under these restrictions it is readily observed that $\forall x \in \X$ and $\forall W \in \W$, we have that $\|Wx\|_{\infty} \leq BR$.

Given a vector $z \in \reals^K$ define the softmax function $\sigma(z): \reals^K \rightarrow  \Delta_K$ as 
\[ [\sigma(z)]_{k} \defeq \frac{e^{[z]_k}}{\sum_{j \in [K]} e^{[z]_j}} \quad \forall k \in [K].\]
With the above notation we can define the multiclass logistic loss $\ell(\cdot, \cdot)$ function given a vector $z \in \reals^K$ and a distribution $y \in \Delta_K$, as
\[ \ell(z,y) \defeq \sum_{k=1}^{K} - y_k\log([\sigma(z)]_k).\]
In several contexts the label $y$ will denote a class in $[K]$ rather a distribution in $\Delta_K$. In these cases, we will overload notation and use $\ell(z, y)$ to denote the logistic loss: $-\log([\sigma(z)]_y)$.

\paragraph{Online multiclass logistic regression.}

We define the learning problem as follows. The learner makes predictions over $T$ rounds with individual rounds indexed by $t$. In each round $t$, nature provides an input $x_t \in \X$ to the learner and the learner predicts $z_t \in \reals^K$ in response. Nature then provides a label $y_t \in \Delta_K$ and the learner suffers the loss $\ell(z_t, y_t)$. The aim of the learner is to minimize regret with respect to any $W \in \W$, defined as 
\[ \regret(W) \defeq \sum_{t=1}^{T} \ell(z_t, y_t) -  \ell(Wx_t, y_t).\]
For notational convenience, we define 
\[ \ell_t(W) \defeq \ell(Wx_t, y_t).\]
An online learning algorithm for the problem is said to be {\em proper} if it generates a linear predict $W_t \in \W$ before seeing $x_t$, and then setting $z_t = W_t x_t$. 
An {\em improper} predictor is allowed to output any prediction $z_t$ taking into account the observation $x_t$. In particular, the learner we propose will produce a $W_t$ at every step, and predict using $z_t = W_t x_t$, but $W_t$ will depend on the observed $x_t$, and the learner is thus improper.

Finally, we make the convention that the gradient $\nabla \ell_t(W) \in \reals^{Kd}$ and the Hessian $\nabla^2 \ell_t(W) \in \reals^{Kd \times Kd}$ are defined with respect to the canonical vectorization of the the matrix $W$.

\section{Main result}
The algorithm we propose FOLKLORE (Algorithm \ref{alg: ftrl}), is an improper FTRL over a sequence of quadratic surrogate losses.
At any time step, the algorithm decides on an improper regularizer $\phi_t$ which adjusts the prediction matrix depending on $x_t$. The regularizer $\phi_t$ is defined as follows:
\begin{align}
    \label{eq: skew def} \phi_t(W) \defeq \frac{1}{K}\sum_{k=1}^{K} \ell(Wx_t, k) + \langle \overrightarrow{W}, B_t  \rangle.
\end{align}
The vector $B_t \in \reals^{Kd}$ defines a linear term in the regularizer and is key to the analysis. It will be specified in the course of the analysis. We use the following surrogate loss:
\begin{align}
    \hat\ell_t(W) \triangleq \ell_t(W_t) + \ip{\overrightarrow{W}-\overrightarrow{W}_t,\nabla\ell_t(W_t)} + \frac{1}{BR+\ln(K)/2}\norm{\overrightarrow{W}-\overrightarrow{W}_t}^2_{\nabla^2\ell_t(W_t)} \label{eq: surrogate loss}
\end{align}
By definition $\hat\ell_t(W_t)=\ell_t(W_t)$. The following lemma shows that the surrogate loss provides a lower bound over $W\in\mathcal{W}$.
\begin{lemma}
\label{lem: lower quadratic}
For any $x_t$ and $W$ such that $\norm{Wx_t}_\infty\leq BR$, the surrogate function defined in equation~\eqref{eq: surrogate loss} is a lower bound on $\ell_t$, i.e. 
\begin{align*}
    \hat\ell_t(W) \leq \ell_t(W)\,.
\end{align*}
\end{lemma}
The above lemma is the multiclass analogue of Lemma 5 in \cite{jezequel20a} (although Lemma 5 is not stated in the Hessian form above, it can be readily seen to be the same in the appropriate parametrization). The proof is provided in Section \ref{sec:main_proofs}.
\begin{remark}
Given the quadratic form of the surrogate loss, it is natural to suspect that the above lemma is a by-product of a self-concordance like property of logistic regression \citep{bach2010, tran2015composite}. Indeed a similar lower bound can be obtained via generalized self-concordance, however it requires both $\|W_t\|_{2,\infty}$ and $\|W\|_{2,\infty}$ to be bounded. Instead we provide an alternative full proof to a slightly stronger version requiring boundedness of only $W$.  
\end{remark}

\begin{algorithm}
    \caption{\NAME}  \label{alg: ftrl}
    \textbf{input}: regularization $\lambda$, improper regularization $\phi_t$\\
    Set $A_0 = \lambda \cdot\identity_{Kd}$, $G_0 = 0$. \\
    \For{$t=1, \dots,T$}{
        Receive $x_t$ and play $z_t = W_tx_t$ where \\
        \begin{equation}
         W_t = \argmin_{W\in\reals^{d\times k}}\lambda \norm{W}_F^2 +\sum\limits_{s=1}^{t-1}\hat{\ell}_t(W)+\phi_t(W) =  \argmin_{W\in\reals^{d\times k}}\norm{\overrightarrow{W}}^2_{A_{t-1}}+\ip{\overrightarrow{W},G_{t-1}}+\phi_t(W).   \label{eq:folkore-opt}
        \end{equation}
        Receive $y_t$ and suffer loss $\ell_t$.\\
        Update $A_t = A_{t-1} + \frac{1}{BR+\ln(K)/2}\nabla^2 \ell_t(W_t)$. \\
        Update $G_t = G_{t-1} + \nabla \ell_t(W_t)$.
    }
\end{algorithm}

The following theorem provides a regret bound for FOLKLORE. We prove this theorem in the next section. 

\begin{theorem}
\label{thm: main regret}
The regret of Algorithm~\ref{alg: ftrl} with surrogate losses defined by \eqref{eq: surrogate loss}, improper regularization defined by \eqref{eq: skew def} and $\lambda = 2\frac{R}{B}$ against any sequence of contexts is bounded by
\begin{align*}
    \regret = \mathcal{O}((BR+\ln(K))dK\ln(T))\,.
\end{align*}
\end{theorem}
The next theorem specifies the per-step running time of the algorithm. The proof is delegated to the Appendix (Section \ref{sec:runtime}).
\begin{theorem}
\label{thm:runtime}
For every step $t$ of Algorithm \ref{alg: ftrl}, for any $\epsilon > 0$, we can compute a vector $\hat{z}_t$ such that $\|\hat{z}_t - W_tx_t\|^2 \leq \epsilon$ in total time \[O\left( d^2K^3 + K^2\left(1 + \tfrac{R^2}{\lambda}\right) \log\left(\epsilon^{-1}\left(1 + \tfrac{R^2}{\lambda}\right)\right)\right).\]
\end{theorem}

Compared to \citet{foster2018logistic}, our algorithm is significantly faster, but the improved running time comes at the price of a linear dependence on $BR$ instead of logarithmic.

\begin{remark}
Note that, FOLKLORE is stated with $W_t$ being the true minimizer in \eqref{eq:folkore-opt}. However it can be seen from our analysis that it is sufficient to solve the problem upto additive error $\mathrm{poly(1/T)}$. Since Theorem \ref{thm:runtime} shows a logarithmic dependence on this error for running time, exact minimization is assumed for brevity through the rest of the paper.
\end{remark}

\subsection{Analysis}
Our analysis proceeds by first recalling the decomposition of the regret into a sum of \textit{instantaneous regret} terms obtained by \cite{jezequel20a}. While this fact was proved by \cite{jezequel20a}, for completeness we provide a concise proof in Appendix (Section \ref{sec:main_proofs}).
\begin{lemma}
\label{lem:skewed ftrl regret}
Let $\hat{\ell_t}$ be the surrogate quadratic lower bound of the loss function which ensure $\hat\ell_t(W_t)=\ell_t(W_t)$ and $\hat\ell_t(W)\leq \ell_t(W)$ for all $W\in\mathcal{W}$. Then for any $W \in \W$, the regret of Algorithm~\ref{alg: ftrl} for any sequence of regularizers $(\phi_t)_{t=1}^T$ is bounded by 
\begin{align*}
    \regret(W) \leq \lambda \norm{W}_F^2+\frac{1}{4}\sum_{t=1}^T\left[\norm{\nabla \phi_t(W_t)-\nabla \ell_t(W_t)}^2_{ A_t^{-1}}-\norm{\nabla \phi_t(W_t)}^2_{A_{t-1}^{-1}}\right]\,.
\end{align*}
\end{lemma}
Given this decomposition, the goal of selecting a suitable $\phi_t$ is to minimize the \textit{instantaneous regret} defined as 
\[\norm{\nabla \phi_t(W_t)-\nabla \ell_t(W_t)}^2_{ A_t^{-1}}-\norm{\nabla \phi_t(W_t)}^2_{A_{t-1}^{-1}}\,.\]

For the case of $K=2$, \cite{jezequel20a} choose the regularizer $\phi_t(W)$ as the function $\ell(Wx_t, 0) + \ell(Wx_t, 1)$. A simple calculation shows that a similar approach with the natural extension $\sum_k \ell(Wx_t, k)$ does not quite work (see also Section~\ref{sec:hessian-dominance-failure} for a more general discussion). We propose to take an alternative approach. Since we do not know the loss $\ell_t$ ahead of time, we aim at finding the $\phi_t$ that minimizes the worst case instantaneous regret.
For arbitrary values of $\nabla \ell_t(W_t)$, this minimax problem does not have a finite solution because the instantaneous regret is quadratic in $\nabla \ell_t(W_t)$.
The key insight for logistic regression is that the loss function is a convex combination over the finite set $\{\ell(\cdot;x_t,\be_k)\,\vert\, k\in[K]\}$. We will focus over these $K$ ``base'' loss functions first. For these $K$ functions, there are only $K$ possible values of $\nabla \ell_t(W_t)$. 
This allows to recast the quadratic problem in  $\nabla \ell_t(W_t)$ as a linear problem, for which a minimax solution exists. The following lemma summarizes this idea of linearization: 
\begin{lemma}
\label{lem:linearization}
For any $W$, any positive definite matrix $A \succ 0$, any regularization function $\phi_t$, and any $\ell_t(W) \in \{\ell(Wx_t,y)\,\vert\,y\in \Delta_K\}$, the instantaneous regret decomposes as
\begin{align*}
    \norm{\nabla \phi_t(W)-\nabla \ell_t(W)}^2_{A^{-1}}-\norm{\nabla &\phi_t(W)}^2_{A^{-1}} \\
    &\leq -2\ip{A^{-1} \nabla\ell_t(W), \nabla \phi_t(W) - b_t(W,A)} + \operatorname{Tr}(A^{-1} \nabla^2\ell_t(W))\,,
\end{align*}
where 
\[
b_t(W;A) =  \sigma(Wx_t)\otimes x_t - \frac{1}{2}A\diag_\otimes(A^{-1})(\mathbf{1}_K\otimes x_t) \,,
\]
and $\diag_\otimes$ denotes the operator that sets all matrix entries besides those corresponding to the $K$-many $d\times d$ blocks on the diagonal to $0$.
\end{lemma}
The proof is delegated to the Appendix (Section \ref{sec:main_proofs}). Given the above lemma, the natural minimax solution for $\phi_t$ is to ensure for all $W$
\begin{equation}
\label{eq:opt skew condition}
    \nabla \phi_t(W)-b_t(W;A) = 0\,,
\end{equation}
in which case the instantaneous regret reduces to $\operatorname{Tr}(A^{-1} \nabla^2\ell_t(W_t))$. Further note that the LHS in Lemma \ref{lem:linearization} upper bounds the instantaneous regret with both choices $A = A_{t-1}$ and $A=A_t$. In terms of regret, the better choice is selecting $A=A_t$ due to the telescoping sum.
Unfortunately, $A_t$ depends on $W_t$ which makes finding the right $\phi_t$ computationally expensive\footnote{This approach requires solving a fixed point equation.}. Therefore we continue by choosing $A=A_{t-1}$. The following lemma (proved in the Appendix (Section \ref{sec:main_proofs})) shows how to choose $\phi_t$ to satisfy \eqref{eq:opt skew condition}.  
\begin{lemma}
\label{lem: optimal skew}
Setting 
\begin{align*}
    B_t = \frac{1}{K}\mathbf{1}_K\otimes x_t -  \frac{1}{2}A_{t-1}\diag_\otimes(A_{t-1}^{-1})(\mathbf{1}_K\otimes x_t),
\end{align*}
in the definition of $\phi_t$ \eqref{eq: skew def} ensures that for all $W$ we have that $\nabla \phi_t(W)=b_t(W;A_{t-1})$\,.
\end{lemma}
Combining the ideas above we are now ready to prove the main regret bound provided in Theorem \ref{thm: main regret}.
\begin{proof}\textbf{of Theorem~\ref{thm: main regret}}. 
Combining Lemma~\ref{lem:skewed ftrl regret},\ref{lem:linearization},\ref{lem: optimal skew} and the fact that for all $W \in \W$, we have that $\norm{W}_F^2\leq KB^2$ directly leads to the following conclusion that for all $W \in \W$,
\begin{align*}
    \regret(W) \leq \lambda KB^2 + \sum_{t=1}^T\operatorname{Tr}(A_{t-1}^{-1} \nabla^2\ell_t(W_t)).
\end{align*}
We have by definition \[A_{t}=\lambda \bI + \frac{1}{BR+\ln(K)/2}\sum_{s=1}^{t}\nabla^2\ell_s(W_s).\]
For convenience of analysis we define a slightly different series 
\[\tilde{A}_{t}=\frac{\lambda}{2} \bI + \frac{1}{BR+\ln(K)/2}\sum_{s=1}^{t}\nabla^2\ell_s(W_s).\]
In Lemma~\ref{lem: hessian upper bound} in the appendix, we prove that for any $t$
\begin{align*}
    \nabla^2\ell_t(W_t)\preceq R^2 \bI\,,
\end{align*}and therefore by using $\lambda = \frac{2R}{B}$, we have
\[
A_{t-1}\succeq A_{t}-\frac{\lambda}{2}\bI = \tilde{A}_t.
\]
Using the concavity of the $\log \det$ function, which in particular implies for any two positive definite matrices $A,B$,  $\operatorname{Tr}(A^{-1}(A-B))\leq\log\det(A)-\log\det(B) $, we get that
\begin{align*}
\regret & \leq \lambda KB^2 + \sum_{t=1}^T\operatorname{Tr}(A_{t-1}^{-1} \nabla^2\ell_t(W_t)) \\
&\leq \lambda KB^2 + \sum_{t=1}^T\operatorname{Tr}(\tilde{A}_{t}^{-1} \nabla^2\ell_t(W_t)) \\
&= \lambda KB^2 + (BR + \ln(K)/2)\sum_{t=1}^T\operatorname{Tr}\left(\tilde{A}_{t}^{-1} \frac{\nabla^2\ell_t(W_t)}{BR + \ln(K)/2}\right)\\
&\leq \lambda KB^2 + (BR + \ln(K)/2)\sum_{t=1}^T \left(\log \det(\tilde{A}_{t}) - \log \det(\tilde{A}_{t-1})\right) \\
&\leq \lambda KB^2 + (BR + \ln(K)/2)\left(\log \det(\tilde{A}_{T}) - \log \det(\tilde{A}_{0})\right) \\
&= \lambda KB^2 + (BR + \ln(K)/2)\log \det\left(I + \frac{2}{\lambda(BR + \ln(K)/2)}\sum_{t=1}^T \nabla^2 \ell_t(W)\right) \\
&\leq \lambda KB^2 + (BR + \ln(K)/2)\log \det\left((1 + T)\bI \right) \\
&\leq K(2BR+ (BR+\ln(K)/2)(d\ln(1+T)))\\
\end{align*}

\end{proof}

\section{Applications}
\label{sec:applications}
In this section, we show that all applications of logistic loss with linear function classes listed in \citep{foster2018logistic} are solvable with our algorithm.

\subsection{Bandit Multiclass Learning}

The next application of our techniques is the bandit multiclass problem. This problem, first studied by \citet{kakade2008efficient}, considers the protocol of online multiclass learning in Section \ref{sec:prelims} with nature choosing $y_t \in [K]$ in each round, but with the added twist of bandit feedback: in each round, the learner predicts a class $\hat y_t$ sampled from some probability distribution $p_t$ that it computes, and receives feedback on whether the prediction was correct or not, i.e. $\mathbb{I}[\hat y_t \neq y_t]$. The goal is to minimize regret with respect to a reference class of linear predictors, using some appropriate surrogate loss function for the 0-1 loss. In most recent papers, regret bounds for this problem have been replaced by the weaker notion of relative mistake bounds (i.e. a bound on the number of mistakes discounted by the minimal logistic loss of a linear predictor). 

In COLT 2009, \citet{abernethyR09a} posed the open problem of obtaining an {\em efficient} algorithm for the problem with $O(\sqrt{T})$ regret using the multiclass logistic loss as the surrogate loss. \citet{hazan2011newtron} solved the open problem and obtained an algorithm, Newtron,  based on the Online Newton Step algorithm \citep{hazan2007logarithmic} with $\tilde{O}(\sqrt{T})$ regret for the case when norm of the linear predictors scales at most logarithmically in $T$. This result was improved by \citet{foster2018logistic} who developed an algorithm called OBAMA based on the Aggregating Algorithm that obtains $\tilde{O}(\sqrt{T})$ {\em relative mistake} bound across all paramater ranges, but has poor running time. In a beautiful paper, \citet{Hoeven20} managed to get a $O(\sqrt{T})$ relative mistake bound for this problem with excellent running time -- $O(dK)$ per iteration -- by exploiting the gap between the 0-1 loss and the multiclass logistic loss. Relative mistake bounds were also obtained for other surrogate losses: e.g. multiclass hinge loss by \cite{kakade2008efficient}, and for a special family of quadratic loss functions by \cite{beygelzimerOZ17}.

We now give a generic reduction showing that any regret bound for the (full-information) online logistic regression implies a relative mistake bound for bandit multiclass learning. By using \NAME as the base algorithm for this reduction, we obtain another bandit multiclass prediction algorithm with $\tilde{O}(\sqrt{T})$ relative mistake bound with a practical running time. Algorithm~\ref{alg:bandit} gives the reduction. It assumes that it is given an algorithm for online logistic regression, $\Alg$, with the interface $\Alg$.Initialize() to initialize the algorithm, $\Alg.\update(x, y)$ to update its internal state with the example $(x, y)$, and $\Alg.\predict(x)$ to predict logits for class probabilities given an input $x$ with the current internal state. The reduction is very simple and based on an exploration-exploitation split: in each round the algorithm samples a Bernoulli random variable with parameter $\gamma$ (denoted $\text{Ber}(\gamma)$), and if the outcome is $1$, it chooses a class uniformly at random, and if the predicted class is correct, updates $\Alg$. Otherwise, the algorithm predicts using the current internal state of $\Alg$. The following theorem (proved in the Appendix (Section \ref{sec:main_proofs}) provides the regret bound for this reduction. 
\begin{algorithm}
    \textbf{input}: logistic regression algorithm $\Alg$,$\gamma\in[0,1]$\\
    $\Alg$.Initialize()\\
    \For{$t=1, \dots,T$}{
        \If{$\text{Ber}(\gamma)=1$}{
            Predict $\hat y_t\sim \operatorname{unif}_K$.\\
            \If{$\hat y_t=y_t$}{
                $\Alg.\update(x_t,y_t)$
            }
        }
        \Else{
            Predict $\hat y_t\sim \sigma(\Alg.\predict(x_t))$
        }
    }
    \caption{Bandit multiclass reduction}  \label{alg:bandit}
\end{algorithm}

\begin{theorem}
\label{thm: bandit}
For an online multiclass logistic regression algorithm which satisfies for all loss sequences $\ell_t$ and comparators $W\in\mathcal{W}$:
\begin{align*}
    \sum_{t=1}^T \ell_t(W_t)-\ell_t(W) \leq \mathcal{R}(T)\,,
\end{align*}
where $\mathcal{R}(T)$ is a monotonous upper bound function, Algorithm~\ref{alg:bandit} with input $\gamma = \sqrt{\frac{K\mathcal{R}(T)}{T}}$ satisfies
\begin{align*}
    \mathbb{E}\left[\sum_{t=1}^T\mathbb{I}\{\hat y_t \neq y_t\}\right] \leq \min\{\min_{W\in\mathcal{W}}\sum_{t=1}^T\ell(Wx_t,y_t) + \sqrt{KT\mathcal{R}(T)},T\}\,.
\end{align*}
\end{theorem}
Applying Theorem~\ref{thm: bandit} to the regression algorithm of \citet{foster2018logistic} directly recovers their Theorem~6, but our reduction holds in more generality.
Combining the wrapper with Algorithm~\ref{alg: ftrl} directly leads to a relative mistake bound of
\[
\mathbb{E}[\sum_{t=1}^T\hat y_t\neq y_t] \leq \min_{W\in\mathcal{W}}\sum_{t=1}^T\ell_t(W) +\mathcal{O}\left(\min\{T,\sqrt{dK^2(BR+\ln(K))T\ln(BRT)}\}\right)\,.
\]
We can now compare this result with that of \citet{Hoeven20}. Their algorithm gets a relative mistake bound of $O(KBR\sqrt{T})$ with $O(dK)$ running time. Our algorithm gets a better dependence on $BR$, but comes at a cost of $O(\sqrt{d \log(T)})$ factor, and somewhat worse running time.

\subsection{Online Multiclass Boosting}
\label{sec:boosting}
The next application of our techniques is the problem of online multiclass boosting (OMB). We begin by describing first the basic setup of OMB, following closely the notation of \citet{jung2017online}.
In this problem, over a sequence of $T$ rounds indexed by $t$, the learner receives an instance $x_t\in \X$, then selects a class $\hat y_t \in [K]$, and finally observes the true class $y_t \in [K]$. The goal is to minimize the total number of mistakes
$\sum_{t=1}^{T}\mathbb{I}[\hat y_t\neq{}y_t]$.
Assisting the algorithm in this task are $N$ copies of a {\em weak learner} which abide by the following protocol.
At any time step $t=1,\dots,T$, the weak learners receive a feature vector $x_t$ and a cost-matrix $C_t\in\mathcal{C}$ (from some class of matrices $C \in \reals^{K \times K}$),
and returns a prediction $l_t \in [K]$.
Then the true label $y_t$ is revealed to the learner and the weak learner suffers the loss $C_t(y_t,l_t)$.
Following \citet{jung2017online}, we restrict $\mathcal{C}=\{C\in\reals^{K\times K}_+\,\mid\,\forall y\in[K]:\,C(y,y)=0\mbox{ and }\norm{C(y,\cdot)}_1\leq 1\}$.
Weak learners are characterized by the following property.
\begin{definition}[Weak learning condition \citep{jung2017online}]
An environment and a learner outputting $(l_t)_{t=1}^T$ satisfy the multiclass weak learning condition with edge $\gamma$ and sample complexity
$S$ if for all outcomes $(y_t)_{t=1}^T$ and cost matrices $(C_t)_{t=1}^T$ from the set $\mathcal{C}$ adaptively chosen by the environment, we have 
\[
\sum_{t=1}^TC_t(y_t ,l_t) \leq \sum_{t=1}^T \mathbb{E}_{k\sim u_{\gamma,y_t}}[C_t(y_t
, k)] + S\,,
\]
where $u_{\gamma,y}$ is the distribution that samples uniform over $[K]$ with probability $(1-\gamma)$ and otherwise picks the label $y$. 
\end{definition}

Given weak-learners as described above, the aim of multiclass boosting (using multiple instantiation of the weak learner) is to achieve arbitrarily low error/regret (measured by the number of mistakes made by the algorithm), keeping the number of samples received and the number of weak learners instantiated by the algorithm as low as possible. \citet{foster2018logistic} propose AdaBoost.OLM++ (Algorithm~\ref{alg:boosting_multiclass} in the Apendix), an extension of AdaBoost.OLM \citep{jung2017online} and AdaBoost.OL \citep{beygelzimer2015optimal} using their proposed algorithm for multiclass logistic regression as a subroutine. Overall they exhibit an algorithm that achieves a target error rate $\epsilon$, after seeing  $T = \tilde{\Omega}\left( \frac{1}{\epsilon\gamma^2} + \frac{KS}{\epsilon \gamma} \right)$  samples, leveraging at most $N = \tilde{O}\left(\frac{1}{\epsilon \gamma^2}\right)$-many weak learners with edge $\gamma$. However since \cite{foster2018logistic} instantiate a logistic regression instance for every weak learner, their algorithm suffers a very high computational overhead 
per-step per weak-learner due to the use of the Aggregating Algorithm. In this section we demonstrate a more efficient reduction for the online multiclass boosting problem which reduces the problem a \textit{binary} logistic regression problem. As a result we can leverage our proposal \NAME or AIOLI achieving the same (up to $\log(T)$ factors) bounds as \cite{foster2018logistic} but significantly faster running time. 

 We now briefly revisit AdaBoost.OLM++ (detailed description and proofs found in Appendix~\ref{app:boosting}). The high level idea is to let each weak learner $i\in[N]$ iteratively improve the \emph{aggregated} logits predicted. In time step $t$, let $s_t^{i-1}$ denote the aggregation up to weak learner $i-1$, then the cost matrix of the $i$-th weak learner is determined by the gradient of the logistic loss of $s_t^{i-1}$ and its prediction $l_t$ shifts the aggregation by $s_t^i = s_t^{i-1}+\alpha_t^i \be_{l_t}$.
\citet{foster2018logistic} show that one obtains small overall error, as long as the adaptive aggregations $s_t^i = s_t^{i-1}+\alpha_t^i \be_{l_t}$ have small regret compared to the best fixed $\alpha^i$ in hindsight.

We formalize this intuition in the following.
\begin{definition}[Boosting regression problem]
The problem proceeds as a $T$-round problem where at every step $t$, the environment (potentially adversarially and adaptively) picks and reveals to the algorithm, a 
logit $s_t \in\reals^K$ and a label $l_t \in[K]$. The algorithm predicts a vector $\hat{s}_t \in \reals^K$. Then a true label picked by the environment $y_t \in[K]$ is revealed to the algorithm. The aim of the algorithm is to ensure the following type of regret bound:
\[\sum_{t=1}^T-\log([\sigma(\hat{s}_t)]_{y_t}) \leq \min_{\alpha\in[-2,2]}\sum_{t=1}^T -\log([\sigma( s_t+\alpha\be_{l_t}) ]_{y_t}) + \mathcal{R}(T)\,,\]
where $\mathcal{R}(T)$ is a regret function.
\end{definition}

With the above definition, we can generalize the guarantee for the AdaBoost.OLM++ algorithm \citep{foster2018logistic} (stated as Algorithm \ref{alg:boosting_multiclass} in the Appendix (Section \ref{app:boosting})) as in the following proposition (proof provided in the Appendix (Section \ref{app:boosting})).
\begin{proposition}[Variant of Theorem 8 and Proposition 9  \citep{foster2018logistic}]
\label{prop:adaboost}
Given an algorithm solving the boosting regression problem with regret $\mathcal{R}(T)$, the predictions $(\hat y_t)_{t=1}^T$ generated by Algorithm~\ref{alg:boosting_multiclass} satisfy with probability at least $1-\delta$,
\[
\sum_{t=1}^T\mathbb{I}\{\hat y_t\neq y_t\} = \mathcal{O}\left( \frac{T\log(K)+N\mathcal{R}(T)}{\sum_{i=1}^N\gamma_i^2}+\log(N/\delta)  \right)\,.
\]
Suppose further that all weak learners satisfy the weak learning condition with edge $\gamma$ and
sample complexity $S$, then with probability at least $1-\delta$, we have

\[
\sum_{t=1}^T\mathbb{I}\{\hat y_t\neq y_t\} = \mathcal{O}\left( \frac{T\log(K)}{N\gamma^2}+\frac{\mathcal{R}(T)}{\gamma^2}+\frac{KS}{\gamma} +\log(N/\delta) \right)\,.
\]
\end{proposition}
\citet{foster2018logistic} provide a solution to the boosting regression problem with regret $\mathcal{R}(T)=\tilde{\mathcal{O}}(1)$ by recasting it as a regular regression problem, mapping $\alpha$ to $W_\alpha\in\reals^{K\times 2K}$ and expanding the dimension to $d=2K$.
Unlike \NAME, the regret bound of multiclass logistic regression provided by \citet{foster2018logistic} scales with the algebraic dimension of the comparator class $\mathcal{W}$, which is only $1$ instead of $2K^2$.
We show in the next paragraph a significantly more efficient reduction based on \NAME.

\paragraph{Solving the boosting regression.}
We present a suitable reduction of the boosting regression to \emph{binary} regression,  which therefore does not scale with the number of classes $K$.
Algorithm~\ref{alg:boosting regression} first collapses the dimension of the prediction from $K$ classes to a binary decision representing a prediction of whether the predicted label $l_t$ is correct. 
Then it clips the magnitude of the logits to $\log(T)$ to ensure bounded $R$ and runs \NAME\footnote{AIOLI would work equally well, since we only require binary regression.} over the resulting binary regression problem.
Finally it expands the binary result to $K$ classes by predicting all classes $y\neq l_t$ proportional to their input logits $s_{t,y}$.
\begin{algorithm}
\textbf{Input:} \NAME with $K=2,d=2, R=1+\ln(T), B=2$.\\
    \For{$t=1, \dots,T$}{
        Receive $s_{t},l_t$\;\\
        $\tilde s_{t}\gets{}(s_{t,l_t},\quad\log\sum_{k\neq l_t}\exp(s_{t,k}))$\;\\
        $\tilde x_t = (\min\{\log(T),\max\{-\log(T),\frac{1}{2}(\tilde s_{t,1}-\tilde s_{t,2}) \},\quad 1)$\,.\\
        $\zeta_t\gets\NAME.\predict(\tilde x_t)$ \\
        $\hat s_{t,l_t} \gets \zeta_{t,1}$\\
        $\forall k\in[K]\setminus\{l_t\}:\,\hat s_{t,k} \gets s_{t,k}+\zeta_{t,2}-\log\sum_{k'\neq l_t}\exp(s_{t,k'})$\,.\\
        Play $s_t$ and receive true class $y_t$.\\
        $\NAME.\update(\tilde x_t, 1+\mathbb{I}\{y_t\neq l_t\})$
    }
    \caption{Boosting regression with \NAME subroutine}  \label{alg:boosting regression}
\end{algorithm}

\noindent The following theorem provides a bound on the regret of Algorithm \ref{alg:boosting regression}. We provide the proof in the Appendix (Section \ref{app:boosting}).
\begin{theorem}
\label{thm: boosting regression}
Algorithm~\ref{alg:boosting regression} solves the boosting regression problem with
$
\mathcal{R}(T) = \mathcal{O}\prn*{\log^2(T)}
$.
\end{theorem}
Combining Theorem~\ref{thm: boosting regression} with Proposition~\ref{prop:adaboost} directly leads to an efficient algorithm for online multiclass boosting. 
The running time complexity during each time-step of algorithm~\ref{alg:boosting regression} is only $\mathcal{O}(K)$. 
Hence, in contrast to the work of \citet{foster2018logistic}, our regression routine does not dominate the overall running time complexity of AdaBoost.OLM++, where computing the cost matrices $C_i^t$ alone is at least linear in $K$.

\subsection{Online-to-Batch Conversion}

The \NAME algorithm can be used to obtain a predictor for the batch statistical setting via the standard online-to-batch conversion. Given a a sample set of size $T$ drawn i.i.d. from the underlying distribution, this predictor runs \NAME on the sample (with an arbitrary order) and stops at a random time $\tau \in \{1, 2, \ldots, T\}$. Then, given a new input feature vector $x \in \X$, the predicted class probabilities are computed by solving the optimization problem \eqref{eq:folkore-opt} in \NAME at time step $\tau$ using the procedure from Theorem~\ref{thm:runtime}. Standard online-to-batch conversion analysis~\citep{Cesa-BianchiCG04} then implies that the expected excess risk of this predictor over the optimal linear prediction $W$ with $\|W\|_{2,\infty} \leq B$ is bounded by the average regret, i.e. by $\mathcal{O}\left(\frac{(BR+\ln(K))dK\ln(T)}{T}\right)$.


\section{Conclusion and Future work}

In this paper we gave an efficient online multiclass regression algorithm which enjoys logarithmic regret with running time scaling only quadratically in the dimension of the feature vectors, thus answering an open question of \citet{foster2018logistic}. We also showed how to apply this algorithm to the online bandit multiclass prediction and online multiclass boosting problems via new reductions.

One open question remaining in this line of work is the following. The regret bound in \citep{foster2018logistic} scales with the algebraic dimension of $\mathcal{W}$, while ours scales with the worst-case dimension $Kd$. Our main technique, specifically Lemma~\ref{lem:linearization} breaks when one tries to reduce the dependency to the algebraic dimension. Designing an algorithm with similar performance guarantees to ours but with regret scaling with the algebraic dimension of $\mathcal{W}$ would be quite interesting.


\bibliography{main}

\begin{thebibliography}{25}
\providecommand{\natexlab}[1]{#1}
\providecommand{\url}[1]{\texttt{#1}}
\expandafter\ifx\csname urlstyle\endcsname\relax
  \providecommand{\doi}[1]{doi: #1}\else
  \providecommand{\doi}{doi: \begingroup \urlstyle{rm}\Url}\fi

\bibitem[Abernethy and Rakhlin(2009)]{abernethyR09a}
Jacob~D. Abernethy and Alexander Rakhlin.
\newblock {An Efficient Bandit Algorithm for $\sqrt{T}$ Regret in Online
  Multiclass Prediction?}
\newblock In \emph{COLT}, 2009.

\bibitem[Bach(2010)]{bach2010}
Francis Bach.
\newblock Self-concordant analysis for logistic regression.
\newblock \emph{Electron. J. Statist.}, 4:\penalty0 384--414, 2010.
\newblock \doi{10.1214/09-EJS521}.
\newblock URL \url{https://doi.org/10.1214/09-EJS521}.

\bibitem[Berkson(1944)]{Berkson1944}
Joseph Berkson.
\newblock Application of the logistic function to bio-assay.
\newblock \emph{Journal of the American Statistical Association}, 39:\penalty0
  357–--365, 1944.

\bibitem[Beygelzimer et~al.(2011)Beygelzimer, Langford, Li, Reyzin, and
  Schapire]{beygelzimer2011contextual}
Alina Beygelzimer, John Langford, Lihong Li, Lev Reyzin, and Robert Schapire.
\newblock Contextual bandit algorithms with supervised learning guarantees.
\newblock In \emph{Proceedings of the Fourteenth International Conference on
  Artificial Intelligence and Statistics}, pages 19--26. JMLR Workshop and
  Conference Proceedings, 2011.

\bibitem[Beygelzimer et~al.(2015)Beygelzimer, Kale, and
  Luo]{beygelzimer2015optimal}
Alina Beygelzimer, Satyen Kale, and Haipeng Luo.
\newblock Optimal and adaptive algorithms for online boosting.
\newblock In \emph{International Conference on Machine Learning}, pages
  2323--2331. PMLR, 2015.

\bibitem[Beygelzimer et~al.(2017)Beygelzimer, Orabona, and
  Zhang]{beygelzimerOZ17}
Alina Beygelzimer, Francesco Orabona, and Chicheng Zhang.
\newblock {Efficient Online Bandit Multiclass Learning with
  $\tilde{O}(\sqrt{T})$ Regret}.
\newblock In \emph{ICML}, pages 488--497, 2017.

\bibitem[Cesa{-}Bianchi et~al.(2004)Cesa{-}Bianchi, Conconi, and
  Gentile]{Cesa-BianchiCG04}
Nicol{\`{o}} Cesa{-}Bianchi, Alex Conconi, and Claudio Gentile.
\newblock On the generalization ability of on-line learning algorithms.
\newblock \emph{{IEEE} Trans. Inf. Theory}, 50\penalty0 (9):\penalty0
  2050--2057, 2004.
\newblock \doi{10.1109/TIT.2004.833339}.
\newblock URL \url{https://doi.org/10.1109/TIT.2004.833339}.

\bibitem[Foster et~al.(2018)Foster, Kale, Luo, Mohri, and
  Sridharan]{foster2018logistic}
Dylan~J Foster, Satyen Kale, Haipeng Luo, Mehryar Mohri, and Karthik Sridharan.
\newblock Logistic regression: The importance of being improper.
\newblock In \emph{Conference On Learning Theory}, pages 167--208. PMLR, 2018.

\bibitem[Freund and Schapire(1997)]{freund1997decision}
Yoav Freund and Robert~E Schapire.
\newblock A decision-theoretic generalization of on-line learning and an
  application to boosting.
\newblock \emph{Journal of computer and system sciences}, 55\penalty0
  (1):\penalty0 119--139, 1997.

\bibitem[Hazan and Kale(2011)]{hazan2011newtron}
E.~Hazan and S.~Kale.
\newblock Newtron: an efficient bandit algorithm for online multiclass
  prediction.
\newblock In \emph{Advances in Neural Information Processing Systems}, pages
  891--899, 2011.

\bibitem[Hazan(2019)]{elad-oco}
Elad Hazan.
\newblock Introduction to online convex optimization.
\newblock \emph{CoRR}, abs/1909.05207, 2019.
\newblock URL \url{http://arxiv.org/abs/1909.05207}.

\bibitem[Hazan et~al.(2007)Hazan, Agarwal, and Kale]{hazan2007logarithmic}
Elad Hazan, Amit Agarwal, and Satyen Kale.
\newblock Logarithmic regret algorithms for online convex optimization.
\newblock \emph{Machine Learning}, 69\penalty0 (2-3):\penalty0 169--192, 2007.

\bibitem[Hazan et~al.(2014)Hazan, Koren, and Levy]{pmlr-v35-hazan14a}
Elad Hazan, Tomer Koren, and Kfir~Y. Levy.
\newblock Logistic regression: Tight bounds for stochastic and online
  optimization.
\newblock In Maria~Florina Balcan, Vitaly Feldman, and Csaba Szepesvári,
  editors, \emph{Proceedings of The 27th Conference on Learning Theory},
  volume~35 of \emph{Proceedings of Machine Learning Research}, pages 197--209,
  Barcelona, Spain, 13--15 Jun 2014. PMLR.
\newblock URL \url{https://proceedings.mlr.press/v35/hazan14a.html}.

\bibitem[Jung et~al.(2017)Jung, Goetz, and Tewari]{jung2017online}
Young~Hun Jung, Jack Goetz, and Ambuj Tewari.
\newblock Online multiclass boosting.
\newblock In \emph{Proceedings of the 31st International Conference on Neural
  Information Processing Systems}, pages 920--929, 2017.

\bibitem[Jézéquel et~al.(2020)Jézéquel, Gaillard, and Rudi]{jezequel20a}
Rémi Jézéquel, Pierre Gaillard, and Alessandro Rudi.
\newblock Efficient improper learning for online logistic regression.
\newblock In Jacob Abernethy and Shivani Agarwal, editors, \emph{Proceedings of
  Thirty Third Conference on Learning Theory}, volume 125 of \emph{Proceedings
  of Machine Learning Research}, pages 2085--2108. PMLR, 09--12 Jul 2020.
\newblock URL \url{http://proceedings.mlr.press/v125/jezequel20a.html}.

\bibitem[Jézéquel et~al.(2021)Jézéquel, Gaillard, and
  Rudi]{jezequel2021mixability}
Rémi Jézéquel, Pierre Gaillard, and Alessandro Rudi.
\newblock Mixability made efficient: Fast online multiclass logistic
  regression.
\newblock \emph{HAL preprint: hal-03370530}, 2021.

\bibitem[Kakade and Ng(2005)]{kakade2005online}
Sham~M Kakade and Andrew~Y Ng.
\newblock Online bounds for bayesian algorithms.
\newblock In \emph{Advances in neural information processing systems}, pages
  641--648, 2005.

\bibitem[Kakade et~al.(2008)Kakade, Shalev-Shwartz, and
  Tewari]{kakade2008efficient}
Sham~M. Kakade, Shai Shalev-Shwartz, and Ambuj Tewari.
\newblock Efficient bandit algorithms for online multiclass prediction.
\newblock In \emph{Proceedings of the 25th international conference on Machine
  learning}, pages 440--447. ACM, 2008.

\bibitem[Marteau{-}Ferey et~al.(2019)Marteau{-}Ferey, Ostrovskii, Bach, and
  Rudi]{marteau-ferey}
Ulysse Marteau{-}Ferey, Dmitrii Ostrovskii, Francis~R. Bach, and Alessandro
  Rudi.
\newblock Beyond least-squares: Fast rates for regularized empirical risk
  minimization through self-concordance.
\newblock In Alina Beygelzimer and Daniel Hsu, editors, \emph{COLT}, volume~99
  of \emph{Proceedings of Machine Learning Research}, pages 2294--2340. {PMLR},
  2019.

\bibitem[McMahan and Streeter(2012)]{mcmahan2012open}
H~Brendan McMahan and Matthew Streeter.
\newblock Open problem: Better bounds for online logistic regression.
\newblock In \emph{Conference on Learning Theory}, pages 44--1. JMLR Workshop
  and Conference Proceedings, 2012.

\bibitem[Mourtada and Ga{\"{\i}}ffas(2019)]{mourtada}
Jaouad Mourtada and St{\'{e}}phane Ga{\"{\i}}ffas.
\newblock An improper estimator with optimal excess risk in misspecified
  density estimation and logistic regression.
\newblock \emph{CoRR}, abs/1912.10784, 2019.
\newblock URL \url{http://arxiv.org/abs/1912.10784}.

\bibitem[Tran-Dinh et~al.(2015)Tran-Dinh, Li, and Cevher]{tran2015composite}
Quoc Tran-Dinh, Yen-Huan Li, and Volkan Cevher.
\newblock Composite convex minimization involving self-concordant-like cost
  functions.
\newblock In \emph{Modelling, Computation and Optimization in Information
  Systems and Management Sciences}, pages 155--168. Springer, 2015.

\bibitem[van~der Hoeven(2020)]{Hoeven20}
Dirk van~der Hoeven.
\newblock Exploiting the surrogate gap in online multiclass classification.
\newblock In Hugo Larochelle, Marc'Aurelio Ranzato, Raia Hadsell,
  Maria{-}Florina Balcan, and Hsuan{-}Tien Lin, editors, \emph{NeurIPS}, 2020.
\newblock URL
  \url{https://proceedings.neurips.cc/paper/2020/hash/6ce8d8f3b038f737cefcdafcf3752452-Abstract.html}.

\bibitem[Vovk(1998)]{vovk1998game}
Vladimir Vovk.
\newblock A game of prediction with expert advice.
\newblock \emph{Journal of Computer and System Sciences}, 56\penalty0
  (2):\penalty0 153--173, 1998.

\bibitem[Zinkevich(2003)]{zinkevich2003online}
Martin Zinkevich.
\newblock Online convex programming and generalized infinitesimal gradient
  ascent.
\newblock In \emph{Proceedings of the 20th International Conference on Machine
  Learning (ICML-03)}, pages 928--936, 2003.

\end{thebibliography}

\appendix

\section{Preliminaries}

We begin by providing some simple calculations regarding the gradient and Hessian of the multinomial logistic loss functions. For any $x,y$ we have that 
\begin{align*}
&\nabla_W \ell(Wx,y) = (\sigma(Wx) - y)\otimes x\\
&\nabla^2_W \ell(Wx_t,y) = (\diag(\sigma(Wx))-\sigma(Wx)\sigma(Wx)^\top)\otimes (xx^\top)\,.
\end{align*}
Note that the Hessian is independent of $y$. This fact will be useful in the analysis. The following lemma provides a simple bound on the Hessian. 

\begin{lemma}
\label{lem: hessian upper bound}
For any $x,y$, we have that the hessian satisfies
\[\nabla_W^2 \ell(Wx,y)\preceq \|x\|_2^2\cdot\bI_{Kd}\,.\]
\end{lemma}
\begin{proof}
We have
\begin{align*}
    \nabla_W^2 \ell(Wx,y) = (\diag(\sigma(Wx))-\sigma(Wx)\sigma(Wx)^\top)\otimes (xx^\top)\,.
\end{align*}
We have that $xx^\top \preceq \|x\|_2^2 \identity_{d}$. Further, $\sigma(Wx)$ is a probability distribution and hence
$(\diag(\sigma(Wx))-\sigma(Wx)\sigma(Wx)^\top)\preceq \identity_K$. 
Combining these two inequalities completes the proof.
\end{proof}
\section{Omitted proofs}
\label{sec:main_proofs}
\begin{proof}\textbf{of Lemma~\ref{lem:skewed ftrl regret}}
For any $W$, define $\hat{L}_t(W) \defeq \sum_{s=1}^{t} \hat{\ell}_s(W) + \lambda \|W\|_F^2$ with $L_0(W) \defeq \lambda\|W\|_F^2 $. Further define $\hat{W}_t = \argmin_{W \in \reals^{k \times d}} \hat{L}_t(W)$. We now have the following consequence for any $W$, 
\begin{align*}
    \regret(W) &= \sum_{t=1}^T\left[ \ell_t(W_t)-\ell_t(W)\right]\leq \sum_{t=1}^T\left[ \hat\ell_t(W_t)-\hat{\ell}_t(W)\right]\\
    &=\sum_{t=1}^T\left[ \hat\ell_t(W_t)+\hat L_{t-1}(\hat W_{t-1})-\hat L_t(\hat W_t)\right] +\underbrace{\hat L_T(\hat W_T)-\hat L_T(W)}_{\leq 0 \;\;(\text{since $\hat W_T$ minimizes $\hat L_T$})}-\lambda(\norm{\hat W_1}_F^2-\norm{W}_F^2)\\
    &\leq \lambda\norm{W}_F^2 +\sum_{t=1}^T\left[ \hat L_t(W_t)-\hat L_t(\hat W_t)+\hat L_{t-1}(\hat W_{t-1})-\hat L_{t-1}(W_t)\right]\,.
\end{align*}
The first inequality above follows from the fact that $\hat \ell_t(W_t) = \ell_t(W_t)$ and for all $W$, $\hat \ell_t(W_t) \leq \ell_t(W_t)$. The last inequality follows from the definition of $\hat L_t$.

Note that $\hat L_t$ is a quadratic function with hessian $2A_t$.
For any quadratic function $f$ with invertible hessian $A$, it holds that
\begin{align*}
    f(x)-\min_{x'}f(x') = \frac{1}{2}\norm{\nabla f(x)}^2_{A^{-1}}\,.
\end{align*}
Hence 
\begin{align*}
    \hat L_t(W_t)-\hat L_t(\hat W_t)+\hat L_{t-1}(\hat W_{t-1})-\hat L_{t-1}(W_t) = \frac{1}{4}\left(\norm{\nabla\hat L_t(W_t)}^2_{A^{-1}_t}-\norm{\nabla\hat L_{t-1}(W_t)}^2_{A^{-1}_{t-1}}\right)\,.
\end{align*}
By definition of our strategy of selecting $W_t$ in Algorithm \ref{alg: ftrl}, (since the optimization if unconstrained and is for a function which is strongly convex and continuously differentiable), we have that
\begin{align*}
    \nabla \hat L_{t-1}(W_t)+\nabla \phi_t(W_t) = 0\,,
\end{align*}
which directly gives $\nabla \hat L_t(W_t) = \nabla\ell_t(W_t)-\nabla \phi_t(W_t)$ and $\nabla \hat L_{t-1}(W_t) = -\nabla \phi_t(W_t)$.
Plugging these into the equation above and combining everything completes the proof.
\end{proof}

\begin{proof}\textbf{of Lemma~\ref{lem: lower quadratic}} Given $W_t$ and a target $W$ with the guarantee that $\|Wx_t\|_{\infty} \leq BR$, define the following vectors  
\[ \omega = \sigma(Wx_t) \text{ and } \nu = \sigma(W_tx_t).\]
Further denote $\ln(w)=[\ln(w_i)]_{i=1}^K$ as the element-wise logarithm of a vector $w$ and define $\gamma \defeq \frac{1}{RB+\ln(K)/2}$. The lower bound is defined as 
\[\hat\ell_t(W) = \ell_t(W_t) + \ip{\overrightarrow{W}-\overrightarrow{W}_t,\nabla \ell_t(W_t)}+\gamma \norm{\overrightarrow{W}-\overrightarrow{W}}^2_{\nabla^2\ell_t(W_t)},\]
which implies the following,
\begin{align*}
    \hat\ell_t(W) - \ell_t(W) &= -\ip{\ln(\nu),y_t}+\ip{\ln(\omega),y_t}+\ip{(W-W_t)x_t, \nu-y_t} + \gamma\norm{(W-W_t)x_t}^2_{\diag(\nu)-\nu\nu^\top}\\
    &= -\ip{\ln(\nu),y_t}+\ip{\ln(\omega),y_t}+\ip{\ln(\omega) - \ln(\nu), \nu-y_t} + \gamma\norm{\ln(\omega) - \ln(\nu)}^2_{\diag(\nu)-\nu\nu^\top}\\
 &=    \ip{\ln(\omega)-\ln(\nu),\nu} +\gamma\norm{\ln(\omega)-\ln(\nu)}^2_{\diag(\nu)-\nu\nu^\top} \,,
\end{align*}
The second equality above follows using the facts that $\langle \nu - y_t, \mathbf{1}_K\rangle = 0$ and $(\diag(\nu) - \nu\nu^\top)\mathbf{1}_K = 0$.
Denote $\Omega=\{\sigma(Wx)\,\vert\,W\in\mathcal{W},x\in\mathcal{X}\}\subset\Delta_K$ be the set of all possible distributions induced by $W\in\mathcal{W}$, then we have
\begin{align*}
    \hat\ell_t(W)-\ell_t(W) \leq \max_{\omega\in\Omega}\max_{\nu\in\Delta_K}\underbrace{\ip{\ln(\omega)-\ln(\nu),\nu} +\gamma\norm{\ln(\omega)-\ln(\nu)}^2_{\diag(\nu)-\nu\nu^\top}}_{\defeq F(\nu, \omega)}\,.
\end{align*}
To analyse the above we first fix a $\omega \in \Omega$ and analyse the expression. Let $\nu^*$ be the optimum fixing $\omega$. Note that since for all $W \in \W, X \in \X$ we have $\|Wx\|_{\infty} \leq BR$, we conclude that for any $i \in [K]$, we have $\omega_i \in (0,1)$. Further, by the KKT conditions, at the optimal point $\nu^* \in \Delta_K$, there must exist $\lambda\in\reals$ such that
for all $i\in[K]$, one of the following must hold,
\begin{align*}
    &\frac{\partial}{\partial\nu_i}F(\nu^*, \omega) = \lambda\\
    \text{or } &\frac{\partial}{\partial\nu_i}F(\nu^*, \omega) < \lambda \text{ and }\nu^*_i=0\\
    \text{or } &\frac{\partial}{\partial\nu_i}F(\nu^*, \omega) > \lambda \text{ and }\nu^*_i=1\,.
\end{align*}
We can derive a closed form expression for the derivative given by
\begin{align*}
    \frac{\partial}{\partial\nu_i}F(\nu^*,\omega) = \left(\ln\left(\frac{\omega_i}{\nu^*_i}\right)-1\right)\left(1-2\gamma\ip{\ln\left(\omega\right)-\ln\left(\nu^*\right),\nu^*}\right)+\gamma\ln^2\left(\frac{\omega_i}{\nu^*_i}\right)+2\gamma\ln\left(\frac{\omega_i}{\nu^*_i}\right)
\end{align*}
We observe that for any $i \in [k]$,  \[\lim_{\nu^*_i\rightarrow 0}\frac{\partial}{\partial\nu_i}F(\nu^*, \omega)=+\infty.\]
Therefore by the KKT conditions, there cannot by an $i \in [K]$, such $\nu^*_i = 0$. Since $\nu^* \in \Delta_K$, this also implies that there cannot be an $i \in [K]$, such $\nu^*_i = 1$. Therefore by the KKT conditions we must have that there exists a $\lambda \in \reals$ such that for all $i$, 
\[\frac{\partial}{\partial\nu_i}F(\nu^*, \omega) = \lambda\]
The above condition can be re-written as for all $i \in [K]$,
\begin{align*}
     \left(\ln\left(\frac{\omega_i}{\nu^*_i}\right)-1\right)\left(1-2\gamma \mathrm{KL}(\omega, \nu^*)\right)+\gamma\ln^2\left(\frac{\omega_i}{\nu^*_i}\right)+2\gamma\ln\left(\frac{\omega_i}{\nu^*_i}\right) = \lambda 
\end{align*}
Further observe that the LHS above across all $i$, is a fixed quadratic function in $\ln(\frac{w_i}{\nu^*_i})$ (Note that the $\mathrm{KL}$ term is fixed across constants). Therefore to satisfy the above equation, it must be the case that across all $i$, $\ln(\frac{w_i}{\nu^*_i})$ can acquire at most two distinct values. Formally, there must exist two values $\alpha_1, \alpha_2 \in \reals$ and a subset $J \subseteq [K]$ such that the following holds,
\begin{align*}
    \forall i\in J:\, &\frac{w_i}{\nu^*_i} = \alpha_1\\
\forall i\in [K]\setminus J:\, &\frac{w_i}{\nu^*_i} = \alpha_2.
\end{align*}
Using the above define the following quantities, let \[w\defeq\sum_{i\in J}\omega_i, \quad v\defeq\sum_{i\in J}\nu^*_i,\] then we have that
\begin{multline*}
    F(\nu^*,\omega) = f(v,w) \defeq  v\ln\left(\frac{w}{v}\right)+(1-v)\ln\left(\frac{1-w}{1-v}\right) - \\ \gamma\left(\left(v\ln\left(\frac{w}{v}\right)+(1-v)\ln\left(\frac{1-w}{1-v}\right)\right)^2+v\ln^2\left(\frac{w}{v}\right)+(1-v)\ln^2\left(\frac{1-w}{1-v}\right)\right).
\end{multline*}

At this point we have effectively reduced the computation to the one-dimensional case. Note that $v \in [0,1]$. Further since $\omega \in \Omega$, we have that 
\[\ln(w)\geq \ln\left(\frac{\exp(-BR)}{(K-1)\exp(BR)+\exp(-BR)}\right) > \ln\left(\frac{\exp(-2BR)}{K}\right) = -2BR-\ln(K).\]
Therefore we can reduce the problem as follows
\begin{equation}
    \label{eqn:temp2}
    \max_{\omega \in \Omega} \max_{\nu \in \Delta_K} F(\nu, \omega) \leq \sup_{w \in ( \frac{e^{-2BR}}{K},1)} \sup_{v \in [0,1]} f(v, w)
\end{equation}
Next we show that 
\begin{equation}
\label{eqn:temp}
    \sup_{w \in ( \frac{e^{-2BR}}{K},1)} \sup_{v \in [0,1]} f(v, w) \leq 0
\end{equation}
 which completes the proof.
To characterize the sup of $f(v,w)$ over $v$ for a given $w$, we consider two cases.

\noindent \textbf{Case 1:} Firstly, let $v\in\{0,1\}$. In this case, $f(v,w)$ is $\ln(w)\leq0$ and $\ln(1-w)\leq0$ respectively. 

\noindent \textbf{Case 2:} The next case is when $\frac{\partial}{\partial v}f(v,w) = 0$ for some $v \in (0,1)$. The derivative is given by
\begin{align*}
    \frac{\partial}{\partial v}f(v,w) = \left(\ln(\frac{w}{1-w})-\ln(\frac{v}{1-v})\right)\left(1+\gamma(2-(v-\frac{1}{2})(\ln(\frac{w}{1-w})- \ln(\frac{v}{1-v})) \right)\,.
\end{align*}
We will show that $\frac{\partial}{\partial v}f(v,w) = 0$ if and only if $v=w$. The forward direction is immediate. To see the backward direction, firstly, observe that for all $v \in (0,1)$ \[(v-\frac{1}{2})\ln(\frac{v}{1-v}) \geq 0,.\]
Therefore we have that for any $w \in ( \frac{e^{-2BR}}{K},1)$ and any $v \in (0,1)$,
\begin{align*}
    \left(1+\gamma(2-(v-\frac{1}{2})(\ln(\frac{w}{1-w})- \ln(\frac{v}{1-v})) \right)&> 1-\gamma|(v-\frac{1}{2})\ln(\frac{w}{1-w})|\\ & > 1-\frac{1}{2RB+\ln(K)}\left\lvert\ln(\frac{w}{1-w})\right\rvert\\
    &>1+\frac{\min\{\ln(w),\ln(1-w)\}}{2RB+\ln(K)}\geq 0. 
\end{align*}
Therefore we have that 
\begin{align*}
    \frac{\partial}{\partial v}f(v,w) = 0 \Longleftrightarrow  \left(\ln(\frac{w}{1-w})-\ln(\frac{v}{1-v})\right) = 0,
\end{align*}
which happens iff and only if $v=w$. It can now be seen that when $v=w$, $f(v,w) = 0$, which establishes \eqref{eqn:temp}. Using \eqref{eqn:temp2}, we have also established 
\begin{align*}
    \max_{\omega\in\Omega}\max_{\nu\in\Delta_K}F(\nu, \omega)\leq 0.
\end{align*}
This completes the proof.
\end{proof}

\begin{proof}\textbf{of Lemma~\ref{lem:linearization}} To prove Lemma \ref{lem:linearization}, as alluded to before, we first prove a restriction of the lemma for the case when $\ell_t \in \{\ell(Wx_t,\be_k)\,\vert\,k\in [K]\}$.

\begin{lemma}
\label{lem:linearization2}
For any $W$, any p.d. matrix $A \succ 0$, any regularization function $\phi_t$, and any $\ell_t(W) \in \{\ell(Wx_t,y)\,\vert\,y\in \Delta_K\}$, the instantaneous regret decomposes as
\begin{align*}
    \norm{\nabla \phi_t(W)-\nabla \ell_t(W)}^2_{A^{-1}}-\norm{\nabla &\phi_t(W)}^2_{A^{-1}} \\
    &= -2\ip{A^{-1} \nabla\ell_t(W), \nabla \phi_t(W) - b_t(W,A)} + \operatorname{Tr}(A^{-1} \nabla^2\ell_t(W))\,,
\end{align*}
where 
\[
b_t(W;A) =  \sigma(Wx_t)\otimes x_t - \frac{1}{2}A\diag_\otimes(A^{-1})(\mathbf{1}_K\otimes x_t) \,,
\]
and $\diag_\otimes$ denotes the operator that sets all matrix entries besides those corresponding to the $K$-many $d\times d$ blocks on the diagonal to $0$.
\end{lemma}
To use the above lemma first note that for any $W$ 
\[\nabla \ell_t(W) = \sum_{k \in [K]} y_k \nabla \ell(Wx_t, \be_k).\]
Therefore an application of Jensen's inequality implies that 

\begin{align*}
    \norm{\nabla \phi_t(W)-\nabla \ell_t(W)}^2_{A^{-1}}-\norm{\nabla &\phi_t(W)}^2_{A^{-1}} \\
    &\leq \left(\sum_{k=1}^{K} \left([y_t]_k \cdot \norm{\nabla \phi_t(W)-\nabla \ell(Wx_t, \be_k)}^2_{A^{-1}} \right) \right)-\norm{\nabla\phi_t(W)}^2_{A^{-1}} \\
    &= \sum_{k=1}^{K} [y_t]_k \cdot \left(\norm{\nabla \phi_t(W)-\nabla \ell(Wx_t, \be_k)}^2_{A^{-1}}  -\norm{\nabla\phi_t(W)}^2_{A^{-1}}\right).\\
\end{align*}
Further using Lemma \ref{lem:linearization2} we get that
\begin{align*}
    \norm{\nabla \phi_t(W)-\nabla \ell_t(W)}^2_{A^{-1}}&-\norm{\nabla \phi_t(W)}^2_{A^{-1}} \\
    &\leq  \sum_{k=1}^K -2\ip{A^{-1} \nabla\ell(Wx_t, \be_k), \nabla \phi_t(W) - b_t(W,A)} + \operatorname{Tr}(A^{-1} \nabla^2\ell_t(W))
    \\ 
    &\leq -2\ip{A^{-1} \nabla\ell_t(W), \nabla \phi_t(W) - b_t(W,A)} + \operatorname{Tr}(A^{-1} \nabla^2\ell_t(W))\,.
\end{align*}
Therefore all that is left to prove is Lemma \ref{lem:linearization2} which we prove next
\begin{proof}\textbf{of Lemma \ref{lem:linearization2}}
Recall that for $\sigma_t = \sigma(Wx_t)$ and for a label $y$, the gradient and Hessian of the logistic loss are respectively
\begin{align*}
&\nabla \ell(Wx_t,y) = (\sigma_t - \be_y)\otimes x_t\\
&\nabla^2 \ell(Wx_t,y) = (\diag(\sigma_t)-\sigma_t\sigma_t^\top)\otimes (x_tx_t^\top)\,.
\end{align*}
Therefore we begin by analyzing the quadratic part of the term which is given by 

\begin{align*}
    \|\nabla \ell_t(W)\|_{A^{-1}}^2 &=\norm{(\sigma_t - \be_y)\otimes x_t}^2_{A^{-1}}\\
    &=\tr((\diag(\be_y)\otimes x_tx_t^\top)A^{-1})+\ip{A^{-1}(\sigma_t - 2\be_y)\otimes x_t,\sigma_t\otimes x_t}\\
    &=\tr((\diag(\be_y-\sigma_t)\otimes x_tx_t^\top)A^{-1})+2\ip{{A^{-1}}((\sigma_t - \be_y)\otimes x_t),\sigma_t\otimes x_t} + \\
    & \qquad \qquad \qquad \qquad \qquad \qquad \qquad \underbrace{\tr((( \diag(\sigma_t) - \sigma_t\sigma_t^{\top}) \otimes x_tx_t^{\top})A^{-1})}_{= \tr(\nabla^2 \ell_t(W)A^{-1})}.
\end{align*}
Finally, note that
\begin{align*}
    \tr((\diag(\be_y-\sigma_t)\otimes x_tx_t^\top)A^{-1}) &= \tr((\diag(\be_y-\sigma_t)\otimes x_tx_t^\top)\diag_{\otimes}(A^{-1}))\\
    &= \tr((\mathbf{1}_K\otimes x_t)((\be_y-\sigma_t)\otimes x_t)^\top\diag_{\otimes}(A^{-1}))\\
    &= \ip{(\be_y-\sigma_t)\otimes x_t),\diag_{\otimes}(A^{-1})(\mathbf{1}_K\otimes x_t)}\,.
\end{align*}
Combining everything and replacing $(\sigma_t - \be_y)\otimes x_t=\nabla \ell_t(W)$ we get that,
\[\|\nabla \ell_t(W)\|_{A^{-1}}^2 = \ip{A^{-1}\nabla \ell_t(W),-A\diag_{\otimes}(A^{-1})(\mathbf{1}_K\otimes x_t) + 2 \sigma_t \otimes x_t} + \tr(\nabla^2 \ell_t(W)A^{-1})\]
Now noting the following expansion
\begin{align*}
    \norm{\nabla \phi_t(W)-\nabla\ell_t(W)}^2_{A^{-1}}-
    \norm{\nabla \phi_t(W)}^2_{A^{-1}}
    = \|\nabla \ell_t(W)\|_{A^{-1}}^2 - 2 \ip{A^{-1}\nabla \ell_t(W),\nabla \phi_t(W)} \\
\end{align*}
and replacing the above finishes the proof. 

\end{proof}
\end{proof}

\begin{proof}\textbf{of Lemma~\ref{lem: optimal skew}}
Recall that $\nabla\ell(Wx_t, \be_k)=(\sigma_t - \be_k)\otimes x_t$.
Hence
\begin{multline*}
    \nabla \phi_t(W) = (\sigma(Wx_t) - \frac{1}{K}\mathbf{1}_K)\otimes x_t+B_t
    = \sigma(Wx_t)\otimes x_t - \frac{1}{2}A_{t-1}\diag_\otimes(A_{t-1}^{-1})(\mathbf{1}_K\otimes x_t)\\
    =b_t(W;A_{t-1}).
\end{multline*}
\end{proof}

\begin{proof}\textbf{of Theorem~\ref{thm: bandit}}
At any time $t$, with probability $\gamma$  we sample a class uniformly at random.  Otherwise we sample a class from $\sigma_t$ of our algorithm. We update our algorithm only if the class was sampled by the extra exploration and if we selected the correct class.
We have
\begin{align*}
    \mathbb{E}\left[\sum_{t=1}^T\mathbb{I}\left(\hat y_t\neq y_t\right)\right]&
    \leq\mathbb{E}\left[\sum_{t=1}^T(1-\sigma_{t,y_t})\right]+\gamma T\\
    &\leq\mathbb{E}\left[\sum_{t=1}^T-\ln(\sigma_{t,y_t})\right]+\gamma T\\
    &\leq\frac{K}{\gamma}\mathbb{E}\left[\sum_{t=1}^T-\ln(\sigma_{t,y_t})\mathbb{I}\{\hat y_t=y_t\land \text{ explore }\}\right]+\gamma T\\
    &\leq\frac{K}{\gamma}\mathbb{E}\left[\sum_{t=1}^T-\ln(\sigma_{t}(W)_{y_t})\mathbb{I}\{\hat y_t=y_t\land \text{ explore }\}+\operatorname{Reg}_T\right]+\gamma T\\
    &=-\sum_{t=1}^T\ln(\sigma_{t}(W)_{y_t})+\frac{K}{\gamma}\operatorname{Reg}_T+\gamma T
\end{align*}
Plugging in the optimal choice of $\gamma$ completes the proof.
\end{proof}

\section{Online Multiclass Boosting}
\label{app:boosting}
AdaBoost.OLM++ (Algorithm~\ref{alg:boosting_multiclass}) takes $N$ weak learners $(\wl^{i})_{i=1}^N$, which are stateful objects that support the operations $\predict(x,C)$ and $\update(x,C,y)$ based on the weak learner protocol described in section~\ref{sec:boosting}. We use the index $t$ to denote the number of updates of any stateful object. Additionally, we have a copy of an algorithm solving the boosting regression problem $\boostreg$ for each weak learner with operations $\predict(s,l)$ and $\update(s,l,y)$. 
The algorithm maintains $s_t^{i}$, which are the weighted aggregated scores of the first $i$ weak learners.
They are literately updated from $s_t^{i-1}$ given the prediction $l_t^i$ such that the regret $\sum_{t=1}^T\left(\log([\sigma(s_t^i)]_{y_t})-\log([\sigma(s_t^{i-1}+\alpha\be_{l_t^i})]_{y_t})\right)$ is small with respect to some $\alpha\in[-2,2]$.
The predictions $s_t^i$ induce a label $\hat y^i_t=\argmax_{k\in[K]}(s_t^i)_k$, which are treated as expert recommendations. Over these expert predictions, AdaBoost.OLM++ runs the hedge algorithm \citep{freund1997decision} to make its final decision.
Finally, the cost matrices $C_t^i$ are computed in the following way.
Let 
\begin{align*}
    \widehat{C}_t^i(k,y)\defeq \frac{\partial}{\partial z_k}\log([\sigma(z)]_y)|_{z=s_t^{i-1}} = \sigma(s_t^{i-1})_k-\mathbb{I}\{k=y\}\,,
\end{align*}
then $C_t^i$ is the translated and rescaled transformation of $\widehat{C}_{t}^i$ that lies in $\mathcal{C}$:
\begin{align}
    C_{t}^i(k,y)= \frac{1}{K}\left(\widehat{C}_{t}^i(k,y)-\widehat{C}_{t}^i(y,y)\right)\,.\label{eq:boosting_cost_matrix}
\end{align}
\begin{algorithm}[ht!]
\caption{AdaBoost.OLM++ \citep{foster2018logistic}}
\label{alg:boosting_multiclass}
\textbf{Input:} weak learners $\wl^{1}, \ldots, \wl^{N}$, boosting regression algorithm $\boostreg$ \;\\
For all $i\in[N]$, set $v_{1}^{i}\gets{} 1$, initialize weak learner $\wl_{1}^{i}$, and initialize copy of boosting regression algorithm $\boostreg_{1}^{i}$.\\
\For{$t=1,\ldots,n$}{
Receive instance $x_{t}$.\\
$s_{t}^{0} \gets{} 0\in\reals^{K}$.\\
\For{$i=1,\ldots,N$}{
Compute cost matrix $C_{t}^{i}$ from $s_{t}^{i-1}$ using \pref{eq:boosting_cost_matrix}.\\
$l_{t}^{i}\gets{}\wl_{t}^{i}.\predict(x_{t}, C_{t}^{i})$.\\
$s_{t}^{i}\gets{}\boostreg_{t}^{i}.\predict(s_{t}^{i-1},l_t^i)$.\\
$\hat{y}_{t}^{i}\gets{}\argmax_{k}s_{t}^{i}(k)$.\\
}
Sample $i_{t}$ with $\Pr(i_{t}=i)\propto{}v_{t}^{i}$. \label{line:sample}\\
Predict $\hat{y}_{t}=\hat{y}_{t}^{i_t}$ and receive true class $y_{t}\in[K]$.
\For{$i=1,\ldots,N$}{
$\wl_{t+1}^{i}\gets{}\wl_{t}^{i}.\update(x_t, C_{t}^{i}, y_t)$.\\
$\boostreg_{t+1}^{i}\gets{}\boostreg_{t}^{i}.\update(s_t^{i-1},l_t^i, y_t)$.\\
$v_{t+1}^{i}\gets{}v_{t}^{i}\cdot\exp(-\mathbb{I}\{\hat{y}_{t}^{i}\neq{}y_t\})$.\label{line:multiplicative_weights}
}
}
\end{algorithm}
The following proof is included for completeness and follows up to minor modifications the work of \cite{foster2018logistic}.
\begin{proof}\textbf{of Proposition~\ref{prop:adaboost}}
We begin with the first part of the proposition.
Denote the number of mistakes of the $i$-th expert (which is the combination of the first $i$ weak learners) by
\[
M_{i} = \sum_{t=1}^{T}\ind\{\hat{y}_{t}^{i}\neq{}y_t\}=\sum_{t=1}^{T}\ind\{\argmax_{k}s_{t}^{i}(k)\neq{}y_t\},
\]
with the convention that $M_{0}=T$. The weights $v_{t}^{i}$ simply implement the multiplicative weights strategy, and so \pref{lem:multiplicative_weights_conc}, which gives a concentration bound based on Freedman's inequality implies that with probability at least $1-\delta$,\footnote{%
Note that previous online boosting works \citep{beygelzimer2015optimal,jung2017online} use a simpler Hoeffding bound at this stage, which picks up an extra $\sqrt{T}$ term. For their results this is not a dominant term, but in our case it can spoil the improvement given by improper logistic regression, and so we use Freedman's inequality to remove it.
}
\begin{equation}
\label{eq:multiplicative_weights_bound}
\sum_{t=1}^{T}\ind\{\hat{y}_{t}\neq{}y_t\} \leq{} 4\min_{i}M_{i} + 2\log(N/\delta).
\end{equation}

Note that if $k^{\star}\defeq\argmax_{k}s_{t}^{i-1}(k)\neq{}y_t$, then $[\sigma(s_{t}^{i-1})]_{k^{\star}}\geq{}[\sigma(s_{t}^{i-1})]_{y_t}$ and $\sigma(s_{t}^{i-1})\in\Delta_{K}$ imply $[\sigma(s_{t}^{i-1})]_{y_t}\leq{}1/2$, which then implies $\sum_{k\neq{}y_t}[\sigma(s_{t}^{i-1})]_{k}\geq{}1/2$ and finally
\begin{equation}
\label{eq:cost_matrix_mistakes}
-\sum_{t=1}^{T}\widehat{C}_{t}^{i}(y_t, y_t) = \sum_{t=1}^T\sum_{k\neq{}y_t}[\sigma(s_{t}^{i-1})]_{k} \geq{} \frac{M_{i-1}}{2}.
\end{equation}
This also holds for $i=1$ because $s_{t}^{0}=0$ and $-\widehat{C}_{t}^{1}(y_t, y_t)=(K-1)/K\geq{}1/2$.

Define the difference between the total loss of the $i$-th and $(i-1)$-th expert to be
\[
\Delta_{i} = \sum_{t=1}^{T}\ell(s_{t}^{i}, y_t) - \ell(s_{t}^{i-1}, y_t).
\]
By the assumption on the employed regression algorithm, we have
\[
\Delta_{i} \leq{} \inf_{\alpha\in\brk*{-2,2}}\brk*{\sum_{t=1}^{T}\ell(\alpha{}\be_{l_{t}^{i}} + s_{t}^{i-1},y_t) - \ell(s_{t}^{i-1}, y_t)} + \mathcal{R}(T).
\]
By \pref{lem:logistic_ub} each term in the sum above satisfies
\[
\ell(\alpha{}\be_{l_{t}^{i}} + s_{t}^{i-1},y_t) - \ell(s_{t}^{i-1}, y_t)
\leq{} \left\{
\begin{array}{ll}
(e^{\alpha}-1)[\sigma(s_{t}^{i-1})]_{l_{t}^{i}} = (e^{\alpha}-1)\widehat{C}_{t}^{i}(y_t, l_{t}^{i}),\quad &l_{t}^{i} \neq{}y_t,\\
(e^{-\alpha}-1)(1-[\sigma(s_{t}^{i-1})]_{y_t}) = -(e^{-\alpha}-1)\widehat{C}_{t}^{i}(y_t, y_t),\quad{} &l_{t}^{i} =y_t.
\end{array}
\right.
\]
With notation $w^{i}=-\sum_{t=1}^{n}\widehat{C}_{t}^{i}(y_t, y_t)$, $c_{+}^{i}=-\frac{1}{w^i}\sum_{t:l_{t}^{i}=y_t}\widehat{C}_{t}^{i}(y_t, y_t)$, and $c_{-}^{i}=\frac{1}{w^i}\sum_{t:l_{t}^{i}\neq{}y_t}\widehat{C}_{t}^{i}(y_t, l_t^{i})$, we rewrite 
\[
\inf_{\alpha\in\brk*{-2,2}}\brk*{\sum_{t=1}^{n}\ell(\alpha{}\be_{l_{t}^{i}} + s_{t}^{i-1},y_t) - \ell(s_{t}^{i-1}, y_t)} 
= w^{i}\cdot\inf_{\alpha\in\brk*{-2,2}}\brk*{(e^{\alpha}-1)c_{-}^{i} + (e^{-\alpha}-1)c_{+}^{i}}.
\]
One can verify that $w^{i}>0$,  $c_{-}^{i}, c_{+}^{i}\geq{}0$, $c_{+}^{i}-c_{-}^{i}=\gamma_{i}\in\brk*{-1,1}$ and $c_{+}^{i} + c_{-}^{i}\leq{}1$. 
By \pref{lem:logistic_inf}, it follows that
\[
w^{i}\cdot\inf_{\alpha\in\brk*{-2,2}}\brk*{(e^{-\alpha}-1)c_{-}^{i} + (e^{\alpha}-1)c_{+}^{i}}
\leq{} -\frac{w^{i}\gamma_{i}^{2}}{2}.
\]

Summing $\Delta_{i}$ over $i\in [N]$, we have
\begin{equation}
\label{eq:delta_sum}
\sum_{t=1}^{T}\ell(s_{t}^{N}, y_t) - \sum_{t=1}^{T}\ell(s_{t}^{0}, y_t)  = \sum_{i=1}^{T}\Delta_i
\leq{} -\frac{1}{2}\sum_{i=1}^{N}w^{i}\gamma_{i}^{2} + N\mathcal{R}(T).
\end{equation}
We lower bound the left hand side as
\[
\sum_{t=1}^{T}\ell(s_{t}^{N}, y_t) - \sum_{t=1}^{T}\ell(s_{t}^{0}, y_t)
\geq{} - \sum_{t=1}^{T}\ell(s_{t}^{0}, y_t) = -T\log(K),
\]
where the inequality uses non-negativity of the logistic loss and the equality is a direct calculation from $s_{t}^{0}=0$.
Next we upper bound the right-hand side of \pref{eq:delta_sum}. Since $w^{i}=-\sum_{t=1}^{T}\widehat{C}_{t}^{i}(y_t, y_t)$, Eq.~\pref{eq:cost_matrix_mistakes} implies
\begin{align*}
-\frac{1}{2}\sum_{i=1}^{N}w^{i}\gamma_{i}^{2} 
\leq  -\frac{1}{4}\sum_{i=1}^{N}M_{i-1}\gamma_{i}^{2} 
\leq{} -\min_{i\in\brk*{N}}M_{i-1}\cdot\frac{1}{4}\sum_{i=1}^{N}\gamma_{i}^{2} 
\leq{} -\min_{i\in\brk*{N}}M_{i}\cdot\frac{1}{4}\sum_{i=1}^{N}\gamma_{i}^{2}.
\end{align*}

Combining our upper and lower bounds on $\sum_{i=1}^{N}\Delta_{i}$ now gives
\begin{equation}
\label{eq:multiclass_boosting_combined}
-T\log(K) \leq{} -\frac{1}{2}\sum_{i=1}^{N}w^{i}\gamma_{i}^{2} + N\mathcal{R}(T)\leq{} -\min_{i\in\brk*{N}}M_{i}\cdot\frac{1}{4}\sum_{i=1}^{N}\gamma_{i}^{2} + N\mathcal{R}(T).
\end{equation}
Rearranging, we have
\[
\min_{i\in\brk*{N}}M_{i} \leq{} O\prn*{\frac{T\log(K)}{\sum_{i=1}^{N}\gamma_{i}^{2}}} + O\prn*{\frac{N\mathcal{R}(T)}{\sum_{i=1}^{N}\gamma_{i}^{2}}}.
\]
Returning to \pref{eq:multiplicative_weights_bound}, this implies that with probability at least $1-\delta$,
\[
\sum_{t=1}^{T}\ind\{\hat{y}_{t}\neq{}y_t\} \leq{} O\prn*{\frac{T\log(K)}{\sum_{i=1}^{N}\gamma_{i}^{2}}} + O\prn*{\frac{N\mathcal{R}(T)}{\sum_{i=1}^{N}\gamma_{i}^{2}}} + 2\log(N/\delta),
\]
which finishes the first part of the proof.

By the definition of the cost matrices, the weak learning condition
\[
\sum_{t=1}^{T}C_{t}^{i}(y_t, l_{t}^{i}) \leq{} \sum_{t=1}^{T}\mathbb{E}_{k\sim{}u_{\gamma, y_t}}\brk*{C_{t}^{i}(y_t, k)} + S
\]
implies
\[
\sum_{t=1}^{T}\widehat{C}_{t}^{i}(y_t, l_{t}^{i}) \leq{} \sum_{t=1}^{T}\mathbb{E}_{k\sim{}u_{\gamma, y_t}}\brk*{\widehat{C}_{t}^{i}(y_t, k)} + KS
\]
Expanding the definitions of $u_{\gamma, y_t}$ and $\widehat{C}_{t}^{i}$, we have
\[
\mathbb{E}_{k\sim{}u_{\gamma, y_t}}\brk*{\widehat{C}_{t}^{i}(y_t, k)} = \prn*{\frac{1-\gamma}{K}}\prn*{[\sigma(s_{t}^{i-1})]_{y_t}-1) + \sum_{k\neq{}y_t}[\sigma(s_{t}^{i-1})]_{k}} + \gamma{}([\sigma(s_{t}^{i-1})]_{y_t}-1) = \gamma{}\widehat{C}_{t}^{i}(y_t, y_t).
\]
So we have
\[
\sum_{t=1}^{T}\widehat{C}_{t}^{i}(y_t, l_{t}^{i}) \leq{} \gamma\sum_{t=1}^{T}\widehat{C}_{t}^{i}(y_t, y_t) + KS, 
\]
or, since $\widehat{C}_{t}^{i}(y_t, y_t)<0$,
\[
\gamma_{i} \geq{} \gamma - \frac{KS}{w^{i}},
\]
where $w^{i}=-\sum_{t=1}^{n} \widehat{C}_{t}^{i}(y_t, y_t)$ as in the first part. 
Since $a \geq b - c$ implies $a^2 \geq b^2 - 2bc$ for non-negative $a, b$ and $c$, 
we further have $\gamma_{i}^{2}\geq{}\gamma^{2}-2\frac{\gamma{}KS}{w^{i}}$.

Returning to the first inequality in \eqref{eq:multiclass_boosting_combined}, the bound we just proved implies
\begin{align*}
-T\log(K) &\leq{} -\frac{1}{2}\sum_{i=1}^{N}w^{i}\gamma^{2} + \gamma{}KSN + N\mathcal{R}(T) \\
&\leq{} -\frac{\gamma^2}{4}\sum_{i=1}^{N}M_{i-1}+ \gamma{}KSN + N\mathcal{R}(T)  \tag{by \pref{eq:cost_matrix_mistakes}}\\
&\leq{} -\min_{i\in\brk*{N}}M_{i} \cdot \frac{\gamma^{2}N}{4} + \gamma{}KSN + N\mathcal{R}(T).
\end{align*}
From here we proceed as in the first part of the proof to get the result.
\end{proof}

\begin{lemma}[Freedman's Inequality \citep{beygelzimer2011contextual}]
\label{lem:freedman}
Let $(Z_t)_{t\leq{}n}$ be a real-valued martingale difference sequence adapted to a filtration $(\mathcal{J}_t)_{t\leq{}n}$ with $\abs{Z_t}\leq{}R$ almost surely. For any $\eta\in[0, 1/R]$, with probability at least $1-\delta$,
\begin{equation}
\label{eq:freedman}
\sum_{t=1}^{n}Z_t \leq{} \eta(e-2)\sum_{t=1}^{n}\mathbb{E}\brk*{Z_t^{2} \mid{} \mathcal{J}_t} + \frac{\log(1/\delta)}{\eta}
\end{equation}
for all $\eta\in\brk*{0, 1/R}$.
\end{lemma}
\begin{lemma}[Lemma 23 \citep{foster2018logistic}]
\label{lem:multiplicative_weights_conc}
With probability at least $1-\delta$, the predictions $(\hat{y}_t)_{t\leq{}n}$ generated by \pref{alg:boosting_multiclass} satisfy
\[
\sum_{t=1}^{T}\ind\crl*{\hat{y}_{t}\neq{}y_t} \leq{} 4\min_{i}\sum_{t=1}^{T}\mathbb{I}\crl*{\hat{y}_{t}^{i}\neq{}y_t} + 2\log(N/\delta).
\]
\end{lemma}
\begin{lemma}[Lemma 24 \citep{foster2018logistic}]
\label{lem:logistic_ub}
The multiclass logistic loss satisfies for any $z \in \reals^K$ and $y\in [K]$,
\[
\ell(z + \alpha{}\be_{l}, y) - \ell(z,y)
\leq{} \left\{
\begin{array}{ll}
(e^{\alpha}-1)[\sigma(z)]_{l},\quad &l \neq{}y,\\
(e^{-\alpha}-1)(1-[\sigma(z)]_{y}),\quad{} &l=y.
\end{array}
\right.
\]
\end{lemma}

\begin{lemma}[\cite{jung2017online}]
\label{lem:logistic_inf}
For any $A,B\geq{}0$ with $A-B\in\brk*{-1, +1}$ and $A+B\leq{}1$,
\[
\inf_{\alpha\in\brk*{-2,2}}\brk*{A(e^{\alpha}-1) + B(e^{-\alpha}-1)} \leq{} -\frac{(A-B)^{2}}{2}.
\]
\end{lemma}

\begin{proof}\textbf{of Theorem~\ref{thm: boosting regression}}
We reduce the problem successively to a regular regression problem for which we can induce Theorem~\ref{thm: main regret}.
Observe that for all $k\neq l_t$, there exists a constant $c_t\in\reals$ such that $\hat s_{t,k}=s_{t,k}+c$. Hence $[\sigma(\hat s_t)]_{k\in[K]\setminus y_t}\propto [\sigma(s_t)]_{k\in[K]\setminus y_t}$.
For the loss, this implies
\[-\log([\sigma(\hat s_t)]_k)+\log([\sigma(s_t+\alpha \be_{l_t})]_k)=-\log(1-[\sigma(\hat s_t)]_{l_t})+\log(1-[\sigma(s_t+\alpha \be_{l_t})]_{l_t})\,.\]
By construction $[\sigma(\hat s_t)]_{l_t}=[\sigma(\zeta_t)]_{1}$ and $[\sigma( s_t+\alpha\be_{l_t})]_{l_t}=[\sigma( \tilde s_t+\alpha \be_{1})]_{1}$, this implies
\[\ell(\hat s_t, y_t)-\ell(s_t+\alpha\be_{l_t}, y_t )= \ell(\zeta_t, 1+\ind\{y_t \neq l_t\})-\ell(\tilde s_t+\alpha\be_1, 1+\ind\{y_t \neq l_t\})\,.\]
Shifting the logits by a constant does not change the distribution or the loss, hence setting
\[
W_\alpha = \begin{pmatrix}1&\frac{1}{2}\alpha\\-1 & -\frac{1}{2}\alpha \end{pmatrix},\, x_t = \begin{pmatrix}\frac{1}{2}(\tilde s_{t,1}-\tilde s_{t,2})\\
1\end{pmatrix}\,,
\]
we have
\[
\ell(\tilde s_t+\alpha\be_1, 1+\ind\{y_t \neq l_t\})= \ell(W_\alpha x_t, 1+\ind\{y_t \neq l_t\})\,.
\]
The algorithm clips the logits to the range $[-\log(T),\log(T)]$ by $\tilde x_{t,1}= \min\{\log(T),\max\{-\log(T),x_{t,1}\}\}$. We bound the induced error by the clipping
\begin{align*}
    &\max_{x_{t,1}\in\reals,\alpha\in[-2,2],y\in\{1,2\}} \ell(W_\alpha \tilde x_t,y)-\ell(W_\alpha x_t,y)\\
    &= \max_{x_{t,1}\in\reals,\alpha\in[-2,2]} \ell(W_\alpha \tilde x_t,1)-\ell(W_\alpha x_t,1)\tag{due to symmetry over the labels}\\
    &= \max_{x_{t,1}\in\reals,\alpha\in[-2,2]}-\log( \frac{\exp(\tilde x_{t,1}+\frac{1}{2}\alpha)}{\exp(\tilde x_{t,1}+\frac{1}{2}\alpha)+\exp(-\tilde x_{t,1}-\frac{1}{2}\alpha)} ) + \log( \frac{\exp(x_{t,1}+\frac{1}{2}\alpha)}{\exp(x_{t,1}+\frac{1}{2}\alpha)+\exp(-x_{t,1}-\frac{1}{2}\alpha)} ) \\
    &= \max_{x_{t,1}\in\reals,\alpha\in[-2,2]}\log(\frac{1+\exp(-2\tilde x_{t,1}-\alpha)}{1+\exp(-2 x_{t,1}-\alpha)})\\
    &= \max_{x_{t,1}\in\reals,\alpha\in[-2,2]}\log(\frac{1+\exp(-2\log(T)-\alpha)}{1+\exp(-2 x_{t,1}-\alpha)}) \tag{$\tilde x_{t,1}<x_{t,1}$ implies $\tilde x_{t,1}=\log(T)$ }\\
    &=\log(1+\exp(-2\log(T)+2))\leq \frac{e^2}{T^2}\,.
\end{align*}
Combining everything up till now yields
\begin{align*}
    \sum_{t=1}^T \ell(\hat s_t, y_t)-\ell(s_t, y_t)&\leq T\frac{e^2}{T^2}+\sum_{t=1}^T \ell(\zeta_t, 1+\ind\{y_t \neq l_t\})-\ell(W_\alpha\tilde x_t, 1+\ind\{y_t \neq l_t\})\\
    &\leq \mathcal{O}(1)+\max_{W:\norm{W}_{2,\infty\leq 2}}\sum_{t=1}^T \ell(\zeta_t, 1+\ind\{y_t \neq l_t\})-\ell(W\tilde x_t, 1+\ind\{y_t \neq l_t\})\,.
\end{align*}
Finally invoking Theorem~\ref{thm: main regret} with $d=2$, $B=2$, $R=\log(T)+1, K=2$ completes the proof.
\end{proof}

\section{Reduction to AIOLI for binary logistic regression}
\label{sec:BinaryAIOLIRedux}
The algorithm AIOLI proposed in \cite{jezequel20a} for the case of $K=2$ proposes to use a regularizer $\phi_t(W) \propto \sum_{y=1}^K\ell(Wx_t,y)$ without the additional bias term $B_t$ introduced by our algorithm. We note that the proposal in \cite{jezequel20a} is for an alternative formulation of binary case and uses a different multiplicative constant (i.e. $BR$ as opposed to $BR + \ln(K)$), however the core of the algorithm is in the choice of the regularizer $\phi_t(W)$. In this section we show that for the binary case (due to the inherent symmetries of the problem) instantiating our algorithm with $B_t = 0$ (as opposed to our proposal) leads to the exact same predictions. We want to emphasize that this emerges from the special structure of the binary case, but does not hold for  $K > 2$.

Let us begin by understanding certain symmetries of the logistic regression problem. Consider splitting the (vectorized) parameter space $\reals^{Kd}$ into two orthogonal spaces $\mathcal{V}\defeq \{\mathbf{1}_K \otimes z\,\mid\, z\in\reals^d\}$ and its orthogonal space $\mathcal{V}^\perp$. Given any $W \in \reals^{K \times d}$ define the projections $W^{\V}$ and $W^{\V^{\perp}}$ obtained by projecting $\overrightarrow{W}$ onto $\V$ and $\V^{\perp}$, and reshaping into $K \times d$ matrices, so that 
\[W = W^{\V} + W^{\V^{\perp}}.\]
Now note that, since $\sigma(z+\gamma\mathbf{1}_K)=\sigma(z)$ for any $\gamma$, and for any $x \in \reals^d$ and $V \in \V$ we have $Vx \propto \mathbf{1}_K$, our predictions only depend on $W^{\V^{\perp}}_t$, i.e.
\[ \sigma(W_tx_t) = \sigma(W^{\V^{\perp}}_tx_t).\]

Furthermore consider any $V\in \mathcal{V}$, and let $z \in \reals^d$ be such that $V=\mathbf{1}_K\otimes z$. For any $x_t\in\reals^d$ and for any $W$ the hessian $\nabla^2\ell_t(W)$ and gradient satisfy
\begin{align*}
&\nabla^2\ell_t(W)V= ((\diag(\sigma_t(W))-\sigma_t(W)\sigma_t(W)^\top)\mathbf{1}_K)\otimes (x_tx_t^\top z_t) = 0\\
&\ip{\nabla\ell_t(W), V} = \ip{\sigma_t(W) - y_t, \mathbf{1}_K}\ip{x_t, z}=0\,.
\end{align*}
It can now be seen that the optimization problem solved at every step can therefore be split by considering $W = U+V$, as
\begin{align*}
    &\min_{W\in\reals^{d\times K}}\lambda\norm{W}_F^2+\sum_{s=1}^{t-1}\hat\ell_s(W)+\phi_t(W)\\
    &=\min_{U\in\mathcal{V}^\perp}\left(\lambda\norm{U}_F^2+\sum_{s=1}^{t-1}\hat\ell_s(U)+\phi_t(U)\right)+\min_{V\in\mathcal{V}} \left(\lambda\norm{V}_F^2+\ip{V,B_t}\right)\,.
\end{align*}
Note that the optimization over $V$ is irrelevant for the eventual predictions. We now show that in the binary case, $B_t\in\mathcal{V}$ which means that for any $U \in \V^{\perp}$, $\ip{B_t, U}=0$, thereby implying that setting it to $0$ does not affect the eventual predictions. 

\paragraph{Binary case.}
Let $K=2$ and through the run of the algorithm denote $\sigma(W_tx_t)_1=p_t$ and therefore we have that $\sigma(W_tx_t)_2=1-p_t$. Denoting $\sigma_t = \sigma(W_tx_t)$, we have
\[\nabla^2 \ell_t(W_t) = (\diag(\sigma_t)-\sigma_t\sigma_t^\top)\otimes(x_tx_t^\top) = p_t(1-p_t)\begin{pmatrix}1&-1\\-1&1
\end{pmatrix}\otimes x_tx_t^\top\,.
\]
Hence setting $M_t=\sum_{s=1}^tp_s(1-p_s)x_sx_s^\top$, we have
\begin{align*}
    A_t = \begin{pmatrix}
    \lambda\identity_{d}+M_t&-M_t\\-M_t&\lambda\identity_{d}+M_t
    \end{pmatrix}\,.
\end{align*}
With some algebra, we can show that
\begin{align*}
    &A_t^{-1} = \begin{pmatrix}
    \lambda^{-1}\identity_{d}-\tilde M_t&\tilde M_t\\\tilde M_t&\lambda^{-1}\identity_{d}-\tilde M_t\end{pmatrix}\\
    &\tilde M_t = \lambda^{-1}(2M_t+\lambda\identity_d)^{-1}M_t\,.
\end{align*}
Hence the bias term for $K=2$ is given by
\begin{align*}
    B_t &= \frac{1}{2}\mathbf{1}_2\otimes x_t - \frac{1}{2}A_t\diag_\otimes(A_t^{-1})\mathbf{1}_2\otimes x_t\\
    &=\frac{1}{2}\mathbf{1}_K\otimes x_t - \frac{1}{2}\mathbf{1}_2\otimes((\identity_d-\lambda\tilde M_t) x_t)\\
    &=\frac{\lambda}{2}\mathbf{1}_K\otimes(\tilde M_t x_t) \in \mathcal{V}\,.
\end{align*} 

\section{Efficient Implementation of Algorithm \ref{alg: ftrl}}
\label{sec:runtime}
In this section we provide a proof of Theorem~\ref{thm:runtime}. We begin by noting that the algorithm can be implemented by computing the vector $z_t = W_tx_t$ and not the full matrix $W_t$.
We show that this can be done efficiently. Note that the algorithm computes $W_t$ as the following
\[W_t = \argmin_{W\in\reals^{K\times d}} \norm{\overrightarrow{W}}^2_{A_{t-1}} + \ip{\overrightarrow{W},G_{t-1}} + \phi_t(W).\]
Since $A_{t-1} \succeq \lambda I_{Kd}$, the above minimization has a unique solution which can be obtained via the following first order optimality condition,
\begin{align*}
 2A_{t-1} \overrightarrow{W}_t + G_{t-1} -\frac{1}{2}A_{t-1}\diag_\otimes(A_{t-1}^{-1})(\mathbf{1}_K\otimes x_t)+\sigma(W_t x_t)\otimes x_t =0\,.
\end{align*}
Rearranging leads to
\begin{align*}
    &\overrightarrow{W}_t = -\frac{1}{2}A_{t-1}^{-1}(G_{t-1}+ \sigma(W_t x_t)\otimes x_t)+\frac{1}{4}\diag_\otimes(A_{t-1}^{-1})(\mathbf{1}_K\otimes x_t).
\end{align*}
Remember that $A_{t-1} \in \reals^{Kd \times Kd}$. For a matrix $M \in \reals^{Kd \times Kd}$, and for $i,j \in [K]$, we denote by $[[M]]_{i,j}$ the $d \times d$ matrix obtained by segmenting $M$ into $K^2$ continguous submatrix blocks of size $d \times d$ in the natural manner, and taking the $(i,j)^{th}$ block. Now define the matrix $\tilde A \in \reals^{K \times K}$, whose $(i,j)^{th}$ entry is given by
\[\left[\tilde A\right]_{i,j} = \frac{1}{2} x_t^{\top}[[A_{t-1}^{-1}]]_{i,j}x_t.\]
Further define $\tilde{g} \in \reals^K$ as a vector whose $k^{th}$ entry is given by
\[\left[\tilde g \right]_k = -\frac{1}{2}\ip{x_t,(A_{t-1}^{-1}G_{t-1})_k}+\frac{1}{4} x_t^{\top}[[A_{t-1}^{-1}]]_{k,k}x_t.\]
We first note the following implications, \[A_{t-1}\succ \lambda \identity_{Kd}\,\Rightarrow\,A_{t-1}^{-1}\prec \lambda^{-1} \identity_{Kd},\]
\[A_{t-1}\preceq (TR^2 + \lambda) \identity_{Kd}\,\Rightarrow\,A_{t-1}^{-1}\succeq (TR^2 + \lambda)^{-1} \identity_{Kd}\]
Furthermore note that for any $v \in \reals^K$, we have that 
\[ v^{\top} \tilde{A} v = (v \otimes x_t)^{\top} A_{t-1}^{-1} (v \otimes x_t).\]
If $\lambda_{\max}(M)$ denotes the largest eigenvalue of a matrix $M$, then the above equation implies that $\lambda_{max}(\tilde{A}) \leq \lambda_{\max}(A_{t-1}^{-1}) \|x_t\|_2^2 = \frac{R^2}{\lambda}$.

In terms of the computation of $z_t$, it can be seen that $z_t$ is the solution of the following equation 
\[z_t = \tilde g - \tilde A\sigma(z_t) \,,\]
which in turn is the first order optimality condition of
\[
z_t=\argmin_{z \in \reals^K} \left[\psi(z) \defeq \frac{1}{2}\norm{z}^2_{\tilde A^{-1}} -\ip{z, \tilde A^{-1}\tilde g}+\log\left(\sum_{k=1}^K\exp(z_k)\right)\right]\,.
\]
We can precondition the above optimization problem as follows: set $\tilde{z} = \tilde{A}^{-1/2}z$ and instead solve the following optimization problem 
\begin{equation}
\label{eqn:optprob2}
    \tilde{z}^*=\argmin_{\tilde z \in \reals^K} \tilde \psi(\tilde z) \defeq \argmin_{\tilde z\in\reals^K}\frac{1}{2}\norm{\tilde z}^2 -\ip{\tilde z, \tilde A^{-1/2}\tilde g}+\log\left(\sum_{k=1}^K\exp([\tilde{A}^{1/2}\tilde z]_k)\right)\,.
\end{equation}
It can be observed that the Hessian of the log term above is bounded by $\lambda_{\max}(\tilde A) \identity \preceq \frac{R^2}{\lambda}\identity_K$ and therefore the smoothness of the above optimization problem is bounded by $\left(1+\frac{R^2}{\lambda}\right)$. In particular this implies that $\tau$ steps of gradient descent, with step size = $\left(\frac{1}{1 + \frac{R^2}{\lambda}}\right)$,  generating the sequence $\{\tilde z_0 \ldots \tilde{z}_{\tau}\}$ ($\tilde z_0$ will be specified momentarily), on the above problem satisfies the following bound
\[ \|\tilde z_{\tau} - \tilde z^*\|^2 \leq \left(1+\frac{R^2}{\lambda}\right) \exp\left(-\frac{\tau}{\left(1+\frac{R^2}{\lambda}\right)} \right) \left(\tilde \psi(\tilde z_0) - \tilde \psi(\tilde z^*) \right).\]
This implies that 
\begin{align*}
    \|\tilde{A}^{1/2}\tilde z_{\tau} - z_t\|^2 &\leq \left(1+\frac{R^2}{\lambda}\right)^2\exp\left(-\frac{\tau}{\left(1+\frac{R^2}{\lambda}\right)} \right)\left(\tilde \psi(\tilde z_0) - \tilde \psi(\tilde z^*) \right).
\end{align*}
$\tilde\psi$ is $1$-strongly convex, and hence by the Polyak-Lojasiewicz (PL) inequality, \[
\tilde \psi(\tilde z_0) - \tilde \psi(\tilde z^*) \leq \frac{1}{2}\norm{\nabla\tilde\psi(\tilde z_0)}_2^2\,.
\]
Setting $\tilde{z}_0 = \tilde A^{-1/2}\tilde g$ yields
\[\norm{\nabla\tilde\psi(\tilde z_0)}_2^2=\norm{\tilde A^{1/2}\sigma(A^{1/2}\tilde{z}_0)}_2^2\leq \left(1+\frac{R^2}{\lambda}\right)\,.\]
We see that within $\tau$ steps of gradient descent on the problem \eqref{eqn:optprob2} we can obtain a vector $\hat{z}_t = \tilde{A}^{1/2}\tilde{z}_{\tau}$ such that 
\begin{align*}
    \|\hat{z}_t - z_t\|^2 
    & \leq \left(1+\frac{R^2}{\lambda}\right)^3\exp\left(-\frac{\tau}{\left(1+\frac{R^2}{\lambda}\right)} \right).
\end{align*}
Therefore setting $\tau = \left(1+\frac{R^2}{\lambda}\right) \log\left( \epsilon^{-1} \left(1+\frac{R^2}{\lambda}\right)^3 \right)$, we see that the error $\|\hat{z}_t - z_t\|^2 \leq \epsilon$.

In terms of computation, note that maintaining the inverse of $A_{t-1}$ requires $O(d^2K^3)$ time. This can be seen by the fact that the update to $A_{t-1}$ is of rank at most $K$ at every step and the update can be performed via the Sherman-Morrison-Woodbury formula. Having computed $A_{t-1}^{-1}$, it can be seen that the quantities $\tilde{A}, \tilde{g}, \tilde{A}^{1/2}, \tilde{A}^{-1/2}$ can all be computed in time $O(d^2K^2 + K^3)$. Having computed these quantities it is easy to see that every step of gradient descent on \eqref{eqn:optprob2} takes time at most $O(K^2)$. Therefore total running time for computing a vector $\hat z_t$ such that $\|\hat z_t - z_t\|^2 \leq  \epsilon$ in total time 
\[O\left( d^2K^3 + K^2\left(1 + \frac{R^2}{\lambda}\right) \log\left( \epsilon^{-1} \left(1+\frac{R^2}{\lambda}\right)^3 \right)\right).\]

\section{Challenges for Hessian-dominance based analysis}
\label{sec:hessian-dominance-failure}

\newcommand{\Tr}{\operatorname{Tr}}

In this section we describe the difficulty in extending the analysis technique of \citet{jezequel20a} to online logistic regression with more than 2 classes. This analysis is based on a Hessian-dominance technique which works as follows. We start from the instantaneous regret defined from Lemma~\ref{lem:skewed ftrl regret}:
\[\norm{\nabla \phi_t(W_t)-\nabla \ell_t(W_t)}^2_{ A_t^{-1}}-\norm{\nabla \phi_t(W_t)}^2_{A_{t-1}^{-1}}.\]
Since $A_{t} \succeq A_{t-1}$, this can be upper bounded by 
\begin{align*}
  &\norm{\nabla \phi_t(W_t)-\nabla \ell_t(W_t)}^2_{ A_t^{-1}}-\norm{\nabla \phi_t(W_t)}^2_{A_{t}^{-1}} \\
  & = \Tr([(\nabla \phi_t(W_t)-\nabla \ell_t(W_t))(\nabla \phi_t(W_t)-\nabla \ell_t(W_t))^\top - \nabla \phi_t(W_t)\nabla \phi_t(W_t)^\top]A_t^{-1}).  
\end{align*}
Then, if we choose a regularizer $\phi_t$ such that the following {\em Hessian-dominance} condition
\begin{equation}
  (\nabla \phi_t(W_t)-\nabla \ell_t(W_t))(\nabla \phi_t(W_t)-\nabla \ell_t(W_t))^\top - \nabla \phi_t(W_t)\nabla \phi_t(W_t)^\top \preceq c \nabla^2 \ell_t(W_t)  \label{eq:hessian_dominance}
\end{equation}
holds for some constant $c$, then the instantaneous regret can be upper bounded by
\[\Tr(c\nabla^2 \ell_t(W_t) A_t^{-1}).\]
Summing this up from $t=1$ to $T$ yields a harmonic sum that can be bounded by standard techniques by $O(c\log T)$. Note that in this analysis we need $c = \text{poly}(B, R)$ to get polynomial dependence on $B$ and $R$ in the regret bound.

For the binary case $K = 2$, it is easy to check that both of the following choices of $\phi_t$
\[\phi_t(W) = \ell(Wx_t, 1) + \ell(Wx_t, 2) \quad \text{ or } \quad \phi_t(W) = \tfrac{1}{2}\ell(Wx_t, 1) + \tfrac{1}{2}\ell(Wx_t, 2)\]
can be used to satisfy \eqref{eq:hessian_dominance}. \citet{jezequel20a} use the first choice of $\phi_t$ to design AIOLI.

Unfortunately, this elegant Hessian dominance technique breaks down when we have more than 2 classes, at least for regularizers of a certain form which satisfy a certain {\em symmetry} condition (satisfied by the two choices for the binary case above). To describe the issue, it will be convenient to drop the $t$ subscript since it is irrelevant to the analysis. We will assume that the regularizer is of the form $\phi(W) = \psi(Wx)$ for some function $\psi: \reals^K \rightarrow \reals$, and the symmetry condition we need is the following: for any two classes $y \neq y'$ and any $z \in \reals^K$, if $z_y = z_{y'}$, then $\nabla \psi(z)_y = \nabla \psi(z)_{y'}$. 

Let $K = 3$. Consider a regularizer $\phi$ satisfying the form and symmetry conditions above. Assume w.l.o.g. that the true label is $1$. Then the gradients of the loss and regularizer are 
\[ \nabla \ell(W) = (\sigma(Wx) - \be_1) \otimes x \quad \text{ and } \quad \nabla \phi(W) = \nabla \psi(Wx) \otimes x.\]
With some abuse of notation, for the rest of this section we will use $\sigma$ to denote $\sigma(Wx)$ and $\psi'$ to denote $\nabla \psi(Wx)$. Then the Hessian equals
\[(\diag(\sigma) - \sigma\sigma^\top) \otimes (xx^\top).\]
The Hessian dominance condition \eqref{eq:hessian_dominance} reduces to the following:
\begin{equation}
  [(\sigma - \be_1 - \psi')(\sigma - \be_1 - \psi')^\top - \psi'{\psi'}^\top] \otimes (xx^\top) \preceq c(\diag(\sigma) - \sigma\sigma^\top) \otimes (xx^\top).  \label{eq:hessian_dominance_simplified}
\end{equation}
We can now show that the above condition cannot hold uniformly for all $W$ and $x$ such that $\|W\|_{2,\infty} \leq B$ and $\|x\| \leq R$ unless $c = \Omega(\exp(BR))$. In particular, choose 
\[W = \text{diag}(0, 0, B) \quad \text{ and } \quad x = (0, 0, R)^\top,\]
so that $Wx = (0, 0, BR)^\top$. Now let $v = (\be_1 - \be_2) \otimes x$. Then we have $v^\top \sigma = 0$, since $\sigma_1 = \sigma_2$, and $v^\top \psi' = 0$, since $\psi'_1 = \psi'_2$ by the symmetry condition. Using these facts, we have
\[v^\top \left([(\sigma - \be_1 - \psi')(\sigma - \be_1 - \psi')^\top - \psi'{\psi'}^\top] \otimes (xx^\top)\right) v = v^\top [(\be_1\be_1^\top) \otimes (xx^\top)] v = \|x\|_2^4 = R^4.\]
Whereas, we have
\[v^\top \left((\diag(\sigma) - \sigma\sigma^\top) \otimes (xx^\top)\right) v = v^\top \left(\diag(\sigma) \otimes (xx^\top)\right) v = (\sigma_1 + \sigma_2)\|x\|_2^4 = (\sigma_1 + \sigma_2)R^4.\]
Since $\sigma_1 = \sigma_2 = \frac{1}{\exp(BR) + 2}$, for the Hessian dominance condition~\eqref{eq:hessian_dominance_simplified} to hold, we must have $c \geq \frac{1}{2}\exp(BR) + 1$, as claimed.
\end{document}